\documentclass{article} 

\PassOptionsToPackage{table,RGB,dvipsnames}{xcolor}

\usepackage{iclr2026_conference,times}

\usepackage{amsmath,amsfonts,bm}

\def\eqref#1{equation~\ref{#1}}

\def\1{\bm{1}}

\DeclareMathAlphabet{\mathsfit}{\encodingdefault}{\sfdefault}{m}{sl}
\SetMathAlphabet{\mathsfit}{bold}{\encodingdefault}{\sfdefault}{bx}{n}

\def\gA{{\mathcal{A}}}
\def\gB{{\mathcal{B}}}

\def\gD{{\mathcal{D}}}

\def\gL{{\mathcal{L}}}
\def\gM{{\mathcal{M}}}
\def\gN{{\mathcal{N}}}

\def\gS{{\mathcal{S}}}

\def\gX{{\mathcal{X}}}

\def\sR{{\mathbb{R}}}

\usepackage{hyperref}
\usepackage{url}
\usepackage{enumitem}
\usepackage{xcolor-material}
\usepackage{cleveref}
\usepackage{algorithm}
\usepackage{graphicx}
\usepackage{amsthm}
\usepackage[noend]{algpseudocode}
\usepackage{subcaption}
\usepackage[T1]{fontenc}
\usepackage{booktabs}
\usepackage[export]{adjustbox}
\usepackage{soul}

\usepackage{mathtools}

\definecolor{plgray}{HTML}{999999}
\definecolor{bcgray}{HTML}{AEAEAE}
\definecolor{reformblue}{HTML}{335E96}
\definecolor{supportred}{HTML}{C94D2E}
\DeclareMathOperator{\supp}{supp}

\newcommand{\unif}{\mathcal{U}}
\newcommand{\pl}[1]{{\color{plgray} #1}}

\newcommand{\algo}{\texttt{ReFORM}}
\newcommand{\ourname}[0]{\textcolor{MaterialBlue900}{{\algo{}}}}
\newcommand{\baselinename}[1]{\textcolor{MaterialPurple900}{\texttt{#1}}}
\newcommand{\datasetname}[1]{{\textsc{#1}}}
\newcommand{\envname}[1]{{\fontfamily{lmtt}\selectfont#1}}
\renewcommand*{\eqref}[1]{(\ref{#1})}
\newtheorem{proposition}{Proposition}
\newtheorem{remark}{Remark}
\newtheorem{theorem}{Theorem}

\newcommand\extrafootertext[1]{\bgroup
    \renewcommand\thefootnote{\fnsymbol{footnote}}\renewcommand\thempfootnote{\fnsymbol{mpfootnote}}\footnotetext[0]{#1}\egroup
}

\Crefname{equation}{Eq.}{Eq.}

\definecolor{newColor}{HTML}{D80000}

\newcounter{fnmarkcntr}

\title{ReFORM: Reflected Flows for On-support \\ Offline RL via Noise Manipulation}

\iclrfinalcopy

\author{
\begin{tabular}{@{}l@{}}Songyuan Zhang$^\dagger$ \hskip1em Oswin So$^\dagger$ \hskip1em H. M. Sabbir Ahmad$^\ddagger$ \hskip1em Eric Yang Yu$^\dagger$\\
Matthew Cleaveland$^*$ \hskip1em Mitchell Black$^*$ \hskip1em Chuchu Fan$^\dagger$
\end{tabular}
\\
$^\dagger$MIT \hskip1em
$^\ddagger$Boston University \hskip1em
$^*$MIT Lincoln Laboratory
\\
$^\dagger$\texttt{\{szhang21,oswinso,eyyu,chuchu\}@mit.edu} \hskip1em 
$^\ddagger$\texttt{sabbir92@bu.edu}\\ 
$^*$\texttt{\{matthew.cleaveland,mitchell.black\}@ll.mit.edu}
}

\usepackage{caption}

\begin{document}

\maketitle

\begin{abstract}

Offline reinforcement learning (RL) aims to learn the optimal policy from a fixed dataset generated by behavior policies without additional environment interactions.
One common challenge that arises in this setting is the out-of-distribution (OOD) error, which occurs when the policy leaves the training distribution.
Prior methods penalize a statistical distance term to keep the policy close to the behavior policy, but this constrains policy improvement and may not completely prevent OOD actions.
Another challenge is that the optimal policy distribution can be multimodal and difficult to represent.
Recent works apply diffusion or flow policies to address this problem, but it is unclear how to avoid OOD errors while retaining policy expressiveness.
We propose \algo{}, an offline RL method based on flow policies that enforces the less restrictive \emph{support constraint} by construction.
\algo{} learns a behavior cloning (BC) flow policy with a bounded source distribution to capture the support of the action distribution, then optimizes a reflected flow that generates bounded noise for the BC flow while keeping the support, to maximize the performance.
Across $40$ challenging tasks from the OGBench benchmark with datasets of varying quality and using a \textit{constant} set of hyperparameters for all tasks, \algo{} dominates all baselines with \textit{hand-tuned} hyperparameters on the performance profile curves.
\footnote{Project website: \url{https://mit-realm.github.io/reform/}}
\extrafootertext{DISTRIBUTION STATEMENT A. Approved for public release. Distribution is unlimited.}
\end{abstract}

\section{Introduction}

Offline reinforcement learning (RL) trains an optimal policy from a previously collected dataset without interacting with the environment~\citep{levine2020offline}.
This technique is especially useful in domains where large datasets are already available and environment interactions are expensive and potentially unsafe \citep{fu2020d4rl}.
However, there are two major challenges.
First, the lack of online exploration makes the distribution shift especially dangerous. 
That is, for out-of-distribution (OOD) actions not represented in the dataset, the learned $Q$-function can produce overly optimistic estimates that lead the policy astray~\citep{levine2020offline}.
Second, traditional policy classes are typically represented using a unimodal distribution such as a Gaussian \citep{kumar2020conservative,tarasov2023revisiting}, whereas more complex offline datasets and tasks can require multimodal action distributions.

Prior works attempt to address the OOD issue by keeping the learned policy close to the behavior policy
by regularizing a statistical distance to the behavior policy
~\citep{wang2018exponentially, peng2019advantage, mao2023str, kumar2019stabilizing, wu2019behavior}.
However, selecting a distance measurement along with an appropriate regularization weight can be difficult depending on the task and dataset.
Perhaps the most common type of statistical distance used is the Kullback–Leibler (KL) divergence~\citep{wang2018exponentially, peng2019advantage, wu2019behavior, jaques2019way, siegel2020keep, nair2020awac, wang2020critic, kostrikov2022offline, fql_park2025}, which can avoid the OOD issue but can also be too restrictive and produce an overly conservative policy.
For example, if the dataset has low density on the optimal behavior, the KL divergence regularization will encourage the learned policy to be suboptimal.
Similar works~\citep{wu2019behavior, kumar2019stabilizing} have considered alternative statistical distances such as the Wasserstein and MMD distances that have been shown to improve performance on certain tasks.
However, these methods do not completely prevent OOD actions, and the need to choose a regularization weight remains a problem.

To tackle the challenge of multimodal action distributions,
recent works have proposed using diffusion policies \citep{hansen2023idql} and flow policies \citep{fql_park2025} to model complex action distributions in the dataset.
However, it remains unclear how to address the OOD issue with these highly expressive function classes without hurting their expressivity. 

In this work, we propose \textbf{Re}flected \textbf{F}lows for \textbf{O}n-support offline \textbf{R}L via noise \textbf{M}anipulation (\algo{}), an offline RL method that aims to address both above issues by constraining a flow policy using the less restrictive \emph{support constraint}.
Rather than regularizing the learned policy via a statistical distance, we only require the actions produced to stay within the support of the action distribution of the behavior policy.
\algo{} learns a behavior cloning (BC) flow policy from the dataset, and additionally learns a reflected flow \citep{xie2024rfm} noise generator that manipulates the source distribution of the BC policy within its support. 
This approach enables us to \emph{realize the support constraint by construction} without regularization, therefore avoiding the need to specify any regularization weights.
In other words, our method bypasses the hyperparameter sensitivity issue by having \emph{constant hyperparameters}.
To summarize our contributions:

\begin{itemize}[leftmargin=2em,topsep=0pt,partopsep=0pt,parsep=0pt]
    \item We propose \algo{}, a two-stage flow policy that realizes the support constraint by construction and avoids the OOD issue without constraining the policy improvement. 
    \item We propose applying reflected flow to generate constrained multimodal noise for the BC flow policy, thereby mitigating OOD errors while maintaining the multimodal policy. 
    \item Extensive experiments on $40$ challenging tasks with datasets of different qualities demonstrate that, with a \textit{constant set of hyperparameters}, \algo{} dominates all baselines using flow policy structures with the \textit{best hand-tuned hyperparameters} on the performance profile curve.
\end{itemize}

\section{Related Work}

\textbf{Distributional shift mitigation in offline RL. } 
A fundamental challenge of dynamic programming methods in offline RL is the OOD problem, where the learned policy tries to exploit erroneous $Q$-values from extrapolation error \citep{levine2020offline} and generates OOD actions.
Consequently, many offline RL methods have proposed to constrain or penalize the statistical distance between the learned policy and the behavior policy, either with an additional loss term or by regressing to the estimated optimal policy, to mitigate this distribution shift issue.
Examples include using the maximum mean discrepancy (MMD) distance \citep{kumar2019stabilizing},
Wasserstein distance \citep{wu2019behavior} and KL divergence \citep{wang2018exponentially,peng2019advantage,wu2019behavior,jaques2019way,siegel2020keep,nair2020awac,wang2020critic,kostrikov2022offline,fql_park2025}.
One key challenge with these methods is that the amount of regularization is a hyperparameter that needs to be tuned for each task and dataset \citep{ogbench_park2025,fql_park2025} and can significantly affect the method's performance. Moreover, as argued by \citet{kumar2019stabilizing}, constraining the divergence can be too restrictive in cases where optimal actions happen with very low probability under the behavior policy.
Another family of methods uses the \textit{support} of the behavior policy, either by regularizing the policy {\citep{kumar2019stabilizing,wu2022spot,mao2023str,zhang2023constrained}}, or via regularizing the $Q$-function outside the support \citep{kumar2020conservative,lyu2022mildly,mao2023svr,cen2024learning}.
Our work falls in the category of enforcing support constraints on the learned policy.
However, instead of approximating the support constraint by a suitably designed regularization term, our method enforces the support constraint \textit{by construction} by optimizing in the BC policy's (bounded) latent space.

\textbf{Fine-tuning flow-based models for offline RL. }
BC methods using diffusion models \citep{sohl2015deep,ho2020denoising,song2020score} or flow matching \citep{lipman2023flow,liu2022flow,albergo2022building} have seen increasing use in the control and robotics communities \citep{chi2023diffusion,reuss2023goal,pearce2023imitating,wang2023diffusion}.
However, since BC aims to mimic the dataset, its performance is tied to the performance of the behavior policy.
To fix this, one can consider fine-tuning the learned flow-based model to maximize a user-supplied reward function.
Following the success of fine-tuning flow-based models for image generation \citep{uehara2024fine,black2024training,domingo2024adjoint}, fine-tuning has also been applied to the offline RL setting \citep{hansen2023idql,chen2024score,fql_park2025,ding2024consistency,Zhang2025EnergyWeightedFlowMatching}.
However, almost all fine-tuning methods for offline RL tackle the problem of distribution shift with an additional loss term penalizing statistical distance from the behavior policy, with the weight of this term being a sensitive hyperparameter that needs to be tuned for each task and dataset \citep{fql_park2025}.

\textbf{Latent space optimization in generative modeling. }
Instead of fine-tuning the flow model directly, another line of work considers optimizing the distribution in the latent space, i.e., initial noise, of the generative model.
In the context of image generation, methods that optimize the initial noise using either regression \citep{li2025noisear,guo2024initno,zhou2024golden,ahn2024noise,eyring2025noise} or RL \citep{miao2025minimalist} have found success in improving the quality of generated images.
In RL, \citet{singh2021parrot} explored using normalizing flows \citep{dinh2016density} to improve exploration in \textit{online} RL.
Since offline RL was not the focus of their work, they do not restrict the output of their learned latent space policy. Consequently, the policy can output unbounded and potentially OOD samples in the latent space, which is harmful in the offline RL setting.
Recently, \citet{zhou2021plas} and \citet{wagenmaker2025steering} have applied this idea to offline RL for Conditional Variational Autoencoders \citep{sohn2015learning} and diffusion policies, respectively, but they additionally restrict the latent space policy to a fixed action magnitude.
Here, the action magnitude roughly controls how likely the latent action is under the behavior policy, playing a similar role to the statistical distance regularization coefficient in existing offline RL works.
As we will show in \Cref{sec: exp}, the final performance is quite sensitive to this hyperparameter, which varies on different tasks and different datasets.
In contrast, our proposed method does not have any such hyperparameters that play a similar role, removing the need for adapting them each time the environment or dataset changes. 

\section{Preliminaries}\label{sec: preliminaries}

\paragraph{Offline RL.}
Let $\Delta(\gX)$ be the set of probability distributions over space $\gX$, and denote placeholder variables with \pl{gray}. A Markov Decision Process (MDP) is defined by a tuple $\gM = (\gS, \gA, r, \rho_0, P, \gamma)$, where $\gS$ is the state space, $\gA\subseteq\sR^d$ is the $d$-dimensional action space, $r(\pl{s},\pl{a}): \gS\times\gA\to\sR$ is the reward function, $\rho_0\in\Delta(\gS)$ is the initial state distribution, $P(\pl{s'}|\pl{s},\pl{a}): \gS\times\gA\to\Delta(\gS)$ is the transition dynamics, and $\gamma\in[0, 1]$ is the discount factor. Given a dataset of $N$ trajectories $\gD=\{\tau_1, \tau_2, \dots, \tau_N\}$ generated by some \textit{behavior} policy $\pi_\beta(\pl{a}|\pl{s}):\gS\to\Delta(\gA)$, where $\tau_i = (s_0, a_0, s_1, a_1, \dots, s_{H_i}, a_{H_i})$, the goal of offline RL is to find a policy $\pi_\theta(\pl{a}|\pl{s}):\gS\to\Delta(\gA)$ parameterized by $\theta$ that maximizes the expected discounted return $R(\pi_\theta) = \mathbb E_{\tau\sim \rho^{\pi_\theta}(\pl{\tau})}[\sum_{h=0}^H\gamma^h r(s_h, a_h)]$, where $\rho^{\pi_\theta}(\tau) = \rho_0(s_0)\pi_\theta(a_0|s_0)P(s_1|s_0, a_0)\cdots\pi_\theta(a_H|s_H)$. Note that in the offline RL setting, sampling in the environment with policy $\pi_\theta$ is not allowed. 

OOD actions are a key challenge in offline RL \citep{levine2020offline}.
Many actor-critic methods learn the policy-conditioned state-action value function (i.e., $Q$-function) $Q(\pl{s},\pl{a}): \gS\times\gA\to\sR$. For a policy $\pi_\theta$, this is defined as
\begin{equation}
    Q^{\pi_\theta}(s, a) = \mathbb E\left[\sum_{h=0}^H \gamma^h r(s_h, a_h)\mid s_0=s,a_0=a,a_h\sim\pi_\theta(s_h), \forall h\geq1\right],
\end{equation}
corresponding to the expected discounted return obtained by applying action $a$ from state $s$ then following policy $\pi_\theta$.
For parameters $\phi$, $Q_\phi$ is commonly learned with fitted $Q$ evaluation using the TD error
\begin{equation}
    \gL(\phi) = \mathbb{E}_{(s, a, s')\sim\gD, a'\sim \pi_\theta(s')}\left[\left(r(s,a)+\gamma Q^{\pi_\theta}_{\hat\phi}(s',a') - Q^{\pi_\theta}_\phi(s, a)\right)^2\right],
\end{equation}
where $Q^{\pi_\theta}_{\hat\phi}$ is a target network (e.g., with soft parameters updated by polyak averaging \citep{polyak1992acceleration}). However, if the policy $\pi_\theta$ samples OOD actions $a'$, the target $Q^{\pi_\theta}_{\hat\phi}$ can produce an erroneous OOD value and cause the learned policy to incorrectly optimize for the OOD value \citep{levine2020offline}. 
To address this issue, many offline RL methods
regularize the statistical distance between
the learned policy and the behavior policy (e.g., with the KL divergence \citep{peng2019advantage,fujimoto2021a,hansenestruch2023idql} or Wasserstein distance \citep{wu2019behavior,fql_park2025}),
resulting in the following objective for policy improvement:
\begin{equation}\label{eq: density constraint}
    \gL(\theta) = \mathbb{E}_{s\sim\gD, a\sim\pi_\theta(s)}\left[-Q^{\pi_\theta}_\phi(s, a) + \alpha D(\pi_\theta\mid\mid\pi_\beta)\right],
\end{equation}
where $D(\cdot \| \cdot)$ is some statistical distance, e.g., $D_\mathrm{KL}$ for KL divergence or $D_\mathrm{W2}$ for the Wasserstein distance.
However, this regularized objective introduces an additional hyperparameter $\alpha$ that needs to be \textit{hand-tuned} for each experiment \citep{ogbench_park2025,fql_park2025}.

\paragraph{Flow matching.}

Flow matching \citep{lipman2023flow,liu2022flow,albergo2022building} has recently become an increasingly popular way of training flow-based generative models. Given a target distribution $p(\pl{x})\in\Delta(\sR^d)$, flow matching learns a time-dependent velocity field $v(\pl{t}, \pl{x})$ that transforms a simple source distribution $q(\pl{x})$ (e.g. standard Gaussian $\gN(0, I^d)$) at $t=0$ to the target distribution $p(\pl{x})$ at $t=1$. The resulting flow $\psi(\pl{t}, \pl{x}) : [0, 1] \times \sR^d \to \sR^d$, mapping samples from the source $x \sim q$ to the target $\psi(1, x) \sim p$,
is then the solution to the ordinary differential equation (ODE)
\begin{equation}\label{eq: flow}
    \frac{d}{dt}\psi(t, x) = v(\psi(t, x)), \quad \psi(0, x) = x.
\end{equation}
Flow matching is a simple yet powerful technique alternative to denoising diffusion \citep{ho2020denoising}, capable of generating complex multimodal target distributions.

\section{Method}\label{sec: method}

To solve the problem of OOD actions, at any given state $s$,
the chosen action $a$ should be constrained to lie within the \textit{support} $\supp(\pi_\beta(\cdot|s)) \coloneqq \{ a \mid \pi_\beta(a|s)>0 \}$ of the behavior policy $\pi_\beta$.
However, constraining common statistical distances, such as the KL divergence or the Wasserstein distance, theoretically leads to problems from the perspective of support constraints \footnote{by interpreting the constant as a Lagrange multiplier, regularization with a fixed coefficient as in \eqref{eq: density constraint} can be interpreted as equivalently enforcing a constraint \citep{levine2020offline}}.
We provide the following theoretical results with all proofs provided in \Cref{app: proof}.

First, constraining the KL divergence is a sufficient but not necessary condition to enforce support constraints \citep{kumar2019stabilizing,mao2023str}. Formally, we have the following result:

\begin{proposition}\label{thm: KL implies support}
    Given a state $s\in\gS$, for any $\epsilon$ such that $0\leq\epsilon<\infty$, $D_\mathrm{KL}(\pi_\theta(\cdot|s)\mid\mid\pi_\beta(\cdot|s))\leq\epsilon$ implies $\supp(\pi_\theta(\cdot|s))\subseteq\supp(\pi_\beta(\cdot|s))$.
    On the other hand, for any $M > 0$, there exist distributions $\pi_\theta$ and $\pi_\beta$ such that $\supp(\pi_\theta(\cdot|s))\subseteq\supp(\pi_\beta(\cdot|s))$ but $D_\mathrm{KL}(\pi_\theta(\cdot|s)\mid\mid\pi_\beta(\cdot|s))> M$.
\end{proposition}
\Cref{thm: KL implies support} tells us that the KL divergence constraint is more restrictive than the support constraint. This additional restriction has been found to impede the performance improvement of $\pi_\theta$ over $\pi_\beta$ \citep{mao2023str}.
While this issue can be alleviated with a small $\alpha$ in \eqref{eq: density constraint},
in practice, this can result in OOD problems due to estimation errors \citep{levine2020offline}. 

Wasserstein distance is another statistical distance used by previous works. However, constraining the Wasserstein distance cannot enforce support constraints despite 
its strong empirical performance in offline RL \citep{fql_park2025}.
Formally, we have the following result:
\begin{proposition}\label{thm: W2 neq support}
    Given a state $s\in\gS$, suppose that $\supp(\pi_\beta(\cdot|s))\neq\gA$. Then, for any $\epsilon>0$, there exists a policy $\pi_\theta$ such that $\supp(\pi_\theta(\cdot|s))\not\subseteq\supp(\pi_\beta(\cdot|s))$, but $D_\mathrm{W2}(\pi_\theta(\cdot|s)\mid\mid\pi_\beta(\cdot|s))\leq\epsilon$.
\end{proposition}

\begin{figure}[t]
    \centering
    \includegraphics[width=1\linewidth]{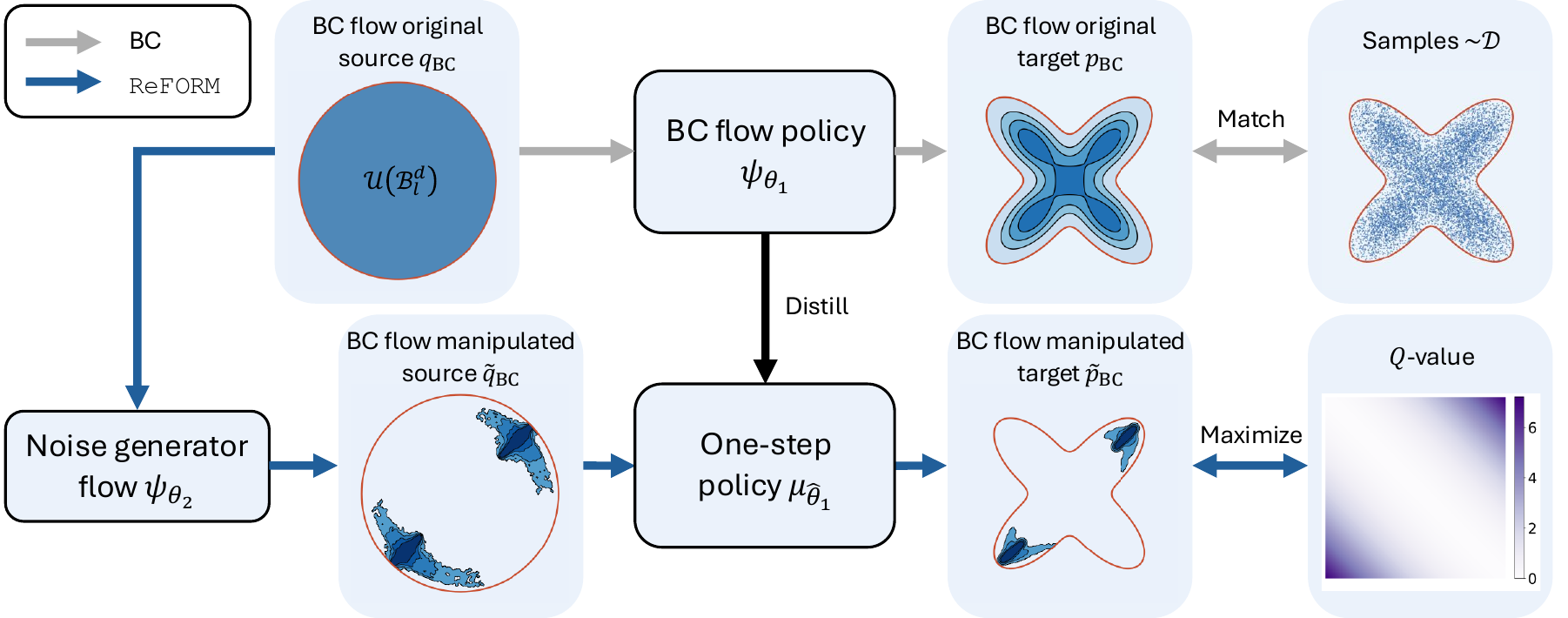}
    \caption{
    \textbf{\algo{} algorithm.} The process with \textcolor{bcgray}{gray} arrows indicates the BC flow policy, learned to transform a simple source distribution $q_\mathrm{BC}=\unif(\gB_l^d)$ to a target distribution $p_\mathrm{BC}$ that matches the dataset $\gD$. 
    The \textcolor{reformblue}{blue} arrows indicate the \algo{} process, where we learn a flow noise generator to generate a manipulated source distribution $\tilde q_\mathrm{BC}$ for the BC policy so that the manipulated target $\tilde p_\mathrm{BC}$ maximizes the $Q$ value while staying inside the support (denoted in \textcolor{supportred}{red}) of the BC policy. 
    }
    \label{fig: algorithm}
\end{figure}

Motivated by the above theoretical challenges of the KL divergence and Wasserstein distance in addressing the issue of OOD actions, we instead consider the following support-constrained policy optimization problem to tackle this issue directly.
\begin{subequations}\label{eq: support-constraint-problem}
    \begin{align}
        \max_\theta\quad & R(\pi_\theta) = \mathbb E_{\tau\sim \rho^{\pi_\theta}}\left[\sum_{h=0}^H\gamma^h r(s_h, a_h)\right], \label{eq:obj} \\
        \mathrm{s.t.}\quad & \supp(\pi_\theta(\cdot|s))\subseteq \supp(\pi_\beta(\cdot|s)),\quad \forall s\in\gS. \label{eq: support-constraint}
    \end{align}
\end{subequations}

Unfortunately, enforcing the support constraint \eqref{eq: support-constraint} is a challenging problem since (i) accurately estimating the $\supp( \pi_\beta(\cdot | s) )$ \citep{grover2018flow}, and (ii) enforcing $\supp(\pi_\theta(\cdot|s))\subseteq \supp(\pi_\beta(\cdot|s))$ given an estimate of $\supp( \pi_\beta(\cdot | s) )$ \citep{zhang2023constrained}, are both nontrivial to solve for.

To tackle these problems, we propose learning a BC flow policy $\psi_{\theta_1}(\pl t, \pl z ; \pl s)$ that transforms a source distribution $q_\mathrm{BC}$ into a state-conditioned target distribution $p_\mathrm{BC}(\pl a|\pl s) \approx \pi_\beta(\pl a| \pl s)$. 
In particular, we use a $q_\mathrm{BC}$ with bounded support
such that $\supp(\pi_\beta)$ can be approximated by the image of $\supp(q_\mathrm{BC})$ under the BC flow.  
One benefit of this approach is that this enables learning a policy that satisfies the support constraints \textit{by construction} by taking advantage of the property that for \textit{any} sample $z \in \supp( q_\mathrm{BC})$ within the (bounded) source distribution's support, $\psi_{\theta_1}(1, z ; s) \in \supp( p_\mathrm{BC}(\cdot|s) ) \approx \supp( \pi_\beta (\cdot|s))$.
Hence, we propose to construct the policy $\pi_\theta$ as the composition of some \textit{noise generator} with the BC flow $\psi_{\theta_1}$.
If the generated noise distribution $\tilde q_\mathrm{BC}$ has the same support as $q_\mathrm{BC}$, i.e., 
\begin{equation} \label{eq:tilde_q_support_constraint}
  \supp(\tilde q_\mathrm{BC})\subseteq \supp(q_\mathrm{BC})  
\end{equation}
then the pushforward of $\tilde q_\mathrm{BC}$ under $\psi_{\theta_1}$ naturally satisfies the support constraints \eqref{eq: support-constraint}.
With the support constraint \eqref{eq: support-constraint} satisfied by construction, solving the support-constrained policy optimization problem \eqref{eq: support-constraint-problem} reduces to performing unconstrained optimization of the objective \eqref{eq:obj}.
\begin{remark}
    This idea of outputting noise is not new.
    Prior works have proposed similar ``noise manipulation/steering'' techniques for fine-tuning diffusion models and flow models \citep{li2025noisear,guo2024initno,miao2025minimalist,wagenmaker2025steering}.
    One \textbf{key} difference is that we choose the source distribution of the flow model to be a distribution with \textbf{bounded} support, which enables better approximation of the support of $\pi_\beta$.
    Moreover, we propose a different form of the noise generator $\tilde{q}_\mathrm{BC}$ than prior works that maintains the high expressivity of flow-based policies.
\end{remark}
We call our method \algo{}, which we summarize in \Cref{fig: algorithm}.
In the following subsections, we elaborate on each of these components in detail.

\subsection{Flow-based behavior policy learning}
\algo{} begins by learning a BC flow policy that transforms the source distribution $q_\mathrm{BC}$ to $p_\mathrm{BC}(\cdot|s)$, which approximates $\pi_\beta(\cdot|s)$.
We choose $q_\mathrm{BC} = \unif({\gB}^d_l)$,
the uniform distribution over the $d$-dimensional hypersphere with radius $l$, so that
\begin{equation} \label{eq: q_BC_support}
    \supp( q_\mathrm{BC} ) = {\gB}^d_l \coloneqq \{ z \in \mathbb{R}^d \mid \| z \| \leq l \}.
\end{equation}
We discuss the choice of $l$ in \Cref{app: additional exp results}. 
To learn the BC flow policy $\psi_{\theta_1}$, 
we learn its corresponding velocity field $v_{\theta_1}(\pl{t}, \pl{z}; \pl{s}):[0,1]\times \gB_l^d\times\mathcal{S}\rightarrow\mathbb{R}^d$ parameterized by $\theta_1$ such that solving the ODE \eqref{eq: flow} gives actions $a = \psi_{\theta_1}(1,z;s)$ for $z\sim q_\mathrm{BC}$. 
We apply a simple linear flow for learning the velocity field following \citet{fql_park2025} with loss 
\begin{equation}
\label{eq:bc_loss}
\mathcal{L}_{\mathrm{BC}}(\theta_1)
= \mathbb{E}_{(s,a)\sim \mathcal{D}, {z \sim \unif({\gB}^d_l)}, {t \sim \unif[0,1]}}
\left[
  \left\|
    v_{\theta_1}\!\big(t,\, x_t; s\big) - (a - z)
  \right\|^2
  \right],
\end{equation}
where $x_t = (1-t)z + ta$ is the linear conditional probability path.

\subsection{Reflected Flow-based Noise Manipulation}\label{sec: reflected flow ng}

A key component in enforcing the support constraints as proposed above is the use of a noise generator with the same support as the BC flow-policy's source distribution $q_\mathrm{BC}$. 
Prior works that apply similar ``noise manipulation'' or ``noise steering'' techniques implement the generated noise $\tilde q_\mathrm{BC}$ as a truncated Gaussian (e.g., by clipping or squashing with $\tanh$). 
However, the use of a \textit{unimodal} $\tilde q_\mathrm{BC}$ severely limits the expressiveness of $\tilde q_\mathrm{BC}$ and thus also that of the resulting learned policy $\pi_\theta$.

One way to improve the expressiveness is by replacing the Gaussian distribution with a flow-based generative model, as has been done with the actions.
We propose to do the same, but to the \textit{noise} instead.
Specifically, we choose to use a flow noise generator $\psi_{\theta_2}(\pl t, \pl w; \pl s):[0, 1]\times\gB_l^d\times\gS\to\gB_l^d$ and denote its associated velocity field as $v_{\theta_2}(\pl{t}, \pl{w}; \pl{s}): [0,1]\times{\gB}^d_l\times\mathcal{S} \to\mathbb{R}^d$.
However, the support of a flow-based generative model is generally unconstrained, which violates our requirement on the support of $\tilde q_\mathrm{BC}$ \eqref{eq:tilde_q_support_constraint}.
To resolve this, we propose to use a \textit{reflected flow} \citep{xie2024rfm}, which can be used to guarantee that samples from $\psi_{\theta_2}$ are contained within $\supp( q_\mathrm{BC} )$ by
considering the following \textit{reflected} ODE \citep{xie2024rfm} instead of \eqref{eq: flow}:
\begin{equation}\label{eq: reflected ode}
    d \psi_{\theta_2}(t, w; s) = v_{\theta_2}(t, \psi_{\theta_2}(t, w; s ); s) dt + dL_t, \quad \psi_{\theta_2}(0, w; s) = w,
\end{equation}
where the reflection term $dL_t$ compensates the outward velocity at $\partial \supp( q_\mathrm{BC} )$ by pushing the motion back to $\supp( q_\mathrm{BC} )$ \citep{xie2024rfm}.

For convenience, let $\mu_{\theta_1}(\pl z; \pl s) = \psi_{\theta_1} (1, \pl z; \pl s)$ and $\mu_{\theta_2}(\pl w; \pl s) = \psi_{\theta_2}(1, \pl w;\pl s)$, and let $\mu_\theta(\pl w; \pl s) = \mu_{\theta_1}( \mu_{\theta_2}( w; s); s)$ denote their composition.
We optimize the noise generator $\psi_{\theta_2}$ to maximize the expected $Q$-value of the learned policy $\mu_\theta$
with the following loss 
\begin{align}\label{eq: actor loss}
    \mathcal{L}_\mathrm{NG}(\theta_2)
    & = \mathbb{E}_{s \sim \mathcal{D}, w \sim \unif(\gB^d_l)}
    \big[
-Q^{\mu_\theta}\left(s, \mu_{\theta_1}(\mu_{\theta_2}(w; s); s)\right) \big],
\end{align}
noting that the parameters of the BC policy $\theta_1$ stay fixed when optimizing $\theta_2$. 

We have yet to specify the reflection term $dL_t$ in \eqref{eq: reflected ode}, as many choices of $dL_t$ constrain the ODE to remain within $\supp( q_\mathrm{BC} )$.
In particular, we wish for the reflection term $dL_t$ to be robust to numerical integration.
Fortunately, $\supp( p_\mathrm{BC} ) = {\gB}^d_l$ being a hypersphere \eqref{eq: q_BC_support}
simplifies this design.
Consider solving the normal ODE \eqref{eq: flow} using the popular Euler method:
\begin{equation}\label{eq: euler}
    z_{k+1} = z_k + v_{\theta_2}(k\Delta t,w;s)\Delta t, \quad k \in \{0, \dots, N - 1\},\quad \psi_{\theta_2}(1,w;s) \gets z_N,
\end{equation}
where $N$ is the number of integration steps, $\Delta t = \frac{1}{N}$, and $z_0 = w$.
For the reflected case \eqref{eq: reflected ode}, we propose modifying the Euler method \eqref{eq: euler} by performing a projection back into the hypersphere after every Euler step.
This gives us the following reflected Euler method
 \begin{equation}\label{eq: euler reflection}
     {z}_{k+1} = \mathbf{1}\{\hat{z}_{k+1} \in \gB_l^d\} \hat{z}_{k+1 } +(1-\mathbf{1}\{\hat{z}_{k+1} \in \gB_l^d\}) \left(\hat{z}_{k+1} - \langle v_{\theta_2}(k\Delta t, w;s)\Delta t, n_{k+1}\rangle n_{k+1}\right),
 \end{equation}
where $\hat{z}_{k+1} = z_k + v_{\theta_2}(k\Delta t,w;s)\Delta t$ follows the original Euler step, $n_k = \frac{\hat{z}_{k}}{\|\hat{z}_{k}\|}$, and $\langle\cdot,\cdot\rangle$ is the inner product.
We then propose to choose $dL_t$ that is defined implicitly by the above procedure. Note that \eqref{eq: euler reflection} has the same complexity as \eqref{eq: euler}, because \eqref{eq: euler reflection} only contains one step projection.

For this to be a valid reflected flow, samples $z$ from the proposed reflected Euler method \eqref{eq: euler reflection} should satisfy the desired support constraints $z \in \supp( q_{\mathrm{BC}} ) = {\gB}^d_l$, which we formally state below.
\begin{theorem}\label{thm: reflected_flow}
The target distribution of the noise generator stays within the support of the original source distribution of the BC policy, i.e., $\supp(\tilde q_\mathrm{BC})\subseteq \supp(q_\mathrm{BC})$.
\end{theorem}

Combining \Cref{thm: reflected_flow} with the ideas from above then allows us to formally prove that the resulting action distribution stays within the support $\supp( p_\mathrm{BC} )$ and hence does not result in OOD actions:
\begin{theorem}\label{thm: support_cnstrnt}
The manipulated target distribution $\tilde p_\mathrm{BC}$ of the BC flow policy
remains within the support of the original BC policy,
    i.e., $\supp(\tilde p_\mathrm{BC})\subseteq \supp(p_\mathrm{BC})$. 
\end{theorem}

\Cref{thm: support_cnstrnt} guarantees that the learned policy provably avoids OOD actions 
\textit{without any regularization terms}.
This avoids the need for costly hyperparameter tuning for each environment and dataset, and also does not impede the potential improvement of the learned policy. 

\subsection{Policy Distillation}

One drawback of our proposed method is that computing the gradient of the actor loss $\nabla_\theta \mathcal{L}_\mathrm{NG}$ \eqref{eq: actor loss} requires computing the gradient $\nabla_{z} \mu_{\theta_1}$, which involves a long backpropagation through time (BPTT) chain since $\mu_{\theta_1}$ is evaluated with Euler integration.
To reduce the computational burden, we follow \citet{fql_park2025} and distill \citep{salimans2022progressive,geng2023one,meanflow} the learned BC flow policy by learning a one-step policy $\hat{\mu}_{\hat\theta_1}(\pl z; \pl s):\gB_l^d\times\gS\to\gA$ parameterized by $\hat\theta_1$ that directly maps the latent variable $z$ to the action $a$ with the following distillation loss:
\begin{equation}
    \mathcal L_\mathrm{Distill}(\hat\theta_1) = \mathbb{E}_{s\sim\gD,z\sim\unif(\gB^d_l)}\left[\|\mu_{\hat\theta_1}(z; s) - \mu_{\theta_1}(z; s)\|^2\right].
\end{equation}

\section{Experiments}\label{sec: exp}

In this section, we conduct experiments to answer the following research questions. Additional details for our implementation, environments, and algorithm hyperparameters, and full results with more ablations are provided in \Cref{app: experiments}.

\begin{enumerate}[topsep=0pt,partopsep=0pt,parsep=0pt,label={(\bfseries Q\arabic*):}]
    \item How does \algo{} perform compared to other offline RL algorithms with flow policies? 
    \item Does \algo{} avoid the OOD issue without limiting the performance improvement?
\item Is it necessary for the BC policy's source distribution to have bounded support?
    \item Is the reflected flow necessary for generating the targeted noise?
    \item How is our design of the reflection term?
    \item Is the distillation of the BC flow policy necessary?
\end{enumerate}

\subsection{Setup}

\textbf{Environments. }
We evaluate \algo{} and the baselines on $40$ tasks from the OGBench offline RL benchmark \citep{ogbench_park2025} designed in $4$ environments, including locomotion tasks and manipulation tasks. 
We use two kinds of datasets, \datasetname{clean} and \datasetname{noisy}. 
The \datasetname{clean} dataset consists of random environment trajectories generated by an expert policy.
The \datasetname{noisy} dataset consists of random trajectories generated by a highly suboptimal and noisy policy. 

\textbf{Baselines. }
We compare \ourname{} with the state-of-the-art offline RL algorithms with flow policies, including Flow Q-Learning (\baselinename{FQL}) \citep{fql_park2025}, Implicit Flow Q-Learning (\baselinename{IFQL}) \citep{fql_park2025}, and Diffusion Steering via RL (\baselinename{DSRL}) \citep{wagenmaker2025steering}. 
Since \baselinename{FQL}'s performance highly depends on the $\alpha$ hyperparameter (\Cref{eq: density constraint}), we consider three variants of \baselinename{FQL}: \baselinename{FQL(M)} uses the $\alpha^*$ that is \textit{hand-tuned} for each environment using the \textsc{clean} dataset by \citet{fql_park2025}, \baselinename{FQL(S)} uses $\alpha = \alpha^*/10$, and \baselinename{FQL(L)} uses $\alpha = 10 \cdot \alpha^*$. 
\baselinename{IFQL} is the flow version of IDQL \citep{hansenestruch2023idql} implemented in \citet{fql_park2025}. For \baselinename{DSRL}, we use the \textit{hand-tuned} noise bound by \citet{wagenmaker2025steering}. 
Note that \ourname{} uses the \textit{same} hyperparameters across \textit{all} tasks. 

\textbf{Evaluation Metrics. }
We run each algorithm with $3$ different seeds for each task and evaluate each converged model on $32$ different initial conditions. We define the \textit{normalized score} for each task as the return normalized by the minimum and maximum returns across all algorithms.

\begin{figure}[t]
    \centering
    \includegraphics[width=\linewidth]{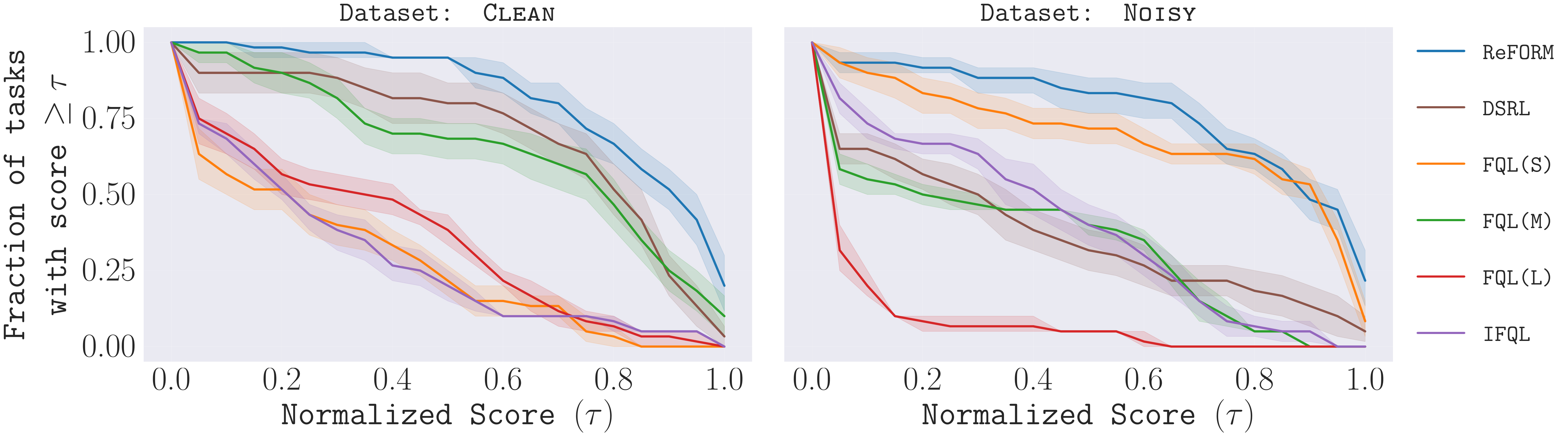}
    \caption{\textbf{Performance profile over \datasetname{clean} and \datasetname{noisy} datasets.} For a given normalized score $\tau$ (x-axis), the performance profile shows the probability that a given method achieves a score $\geq\tau$ (see \citet{agarwal2021deep} for details). On the \datasetname{clean} dataset, \ourname{} achieves greater scores with higher probabilities than all other baselines. The same is true on the \datasetname{noisy} dataset except for a small set of normalized scores around $0.9$ where \ourname{} and \baselinename{FQL(S)} have similar probabilities within the statistical margins.}
    \label{fig: performance-all}
\end{figure}

\subsection{Main results}

\textbf{(Q1): \ourname{} achieves the best overall performance with a constant set of hyperparameters.}
As recommended by \citet{agarwal2021deep}, we plot the performance profile over all tasks with different datasets in \Cref{fig: performance-all}. 
It is clear that \ourname{} achieves the best performance for both the \datasetname{clean} and \datasetname{noisy} datasets. 
For the \datasetname{clean} dataset, \baselinename{DSRL} and \baselinename{FQL(M)} achieve the second and third best respective performance because their hyperparameters are specifically hand-tuned for these environments. 
However, for the \datasetname{noisy} dataset, the performance of both \baselinename{DSRL} and \baselinename{FQL(M)} drops significantly, whereas \baselinename{FQL(S)} becomes the second-best method behind \ourname{}. This highlights the hyperparameter sensitivity of the baseline methods. Moreover, we observe that when the behavior policy performs poorly (i.e., on \datasetname{noisy}), a stronger density-based regularization impedes the ability of the learned policy to improve (see \baselinename{FQL(L)}). 

Importantly, \ourname{} achieves the highest fraction on normalized scores close to $1$, indicating that \ourname{} does not limit the improvement of the learned policy as discussed in \Cref{sec: method}.
The use of a support constraint allows the learned policy to apply any action with the support, including ones that have low density under the behavior policy $\pi_\beta$ 
Therefore, the learned policy does not suffer from a performance upper bound related to the behavior policy. 

\textbf{(Q2): \ourname{} maximizes the performance while avoiding OOD.}
We design a toy example to better visualize and compare the learned policies.
The toy example has a 2-dimensional action space with a $Q$-value that grows when approaching the lower left and the upper right corners (see the $Q$-value plot in \Cref{fig: algorithm}).
We plot the policy distributions of BC and all algorithms in \Cref{fig: toy example}. \ourname{} maximizes performance by reaching both corners while staying within the support of the BC policy.
\baselinename{DSRL} collapses to a single mode in the upper right corner and remains far from the boundaries of the support because the generated noise of \baselinename{DSRL} is unimodal and squashed.
We compare the generated noise in more detail in \Cref{app: additional exp results}, \Cref{fig: viz-noise-toy}. 
\baselinename{IFQL} remains similar to the BC policy because importance sampling is less efficient for finding the maximum.
\baselinename{FQL} faces OOD error due to its use of Wasserstein distance regularization (as discussed in \Cref{thm: W2 neq support}).

\subsection{Ablation studies}\label{sec: ablations}

\begin{figure}[t]
    \centering
\begin{subfigure}{.142\linewidth}
        \centering
        \includegraphics[width=\linewidth]{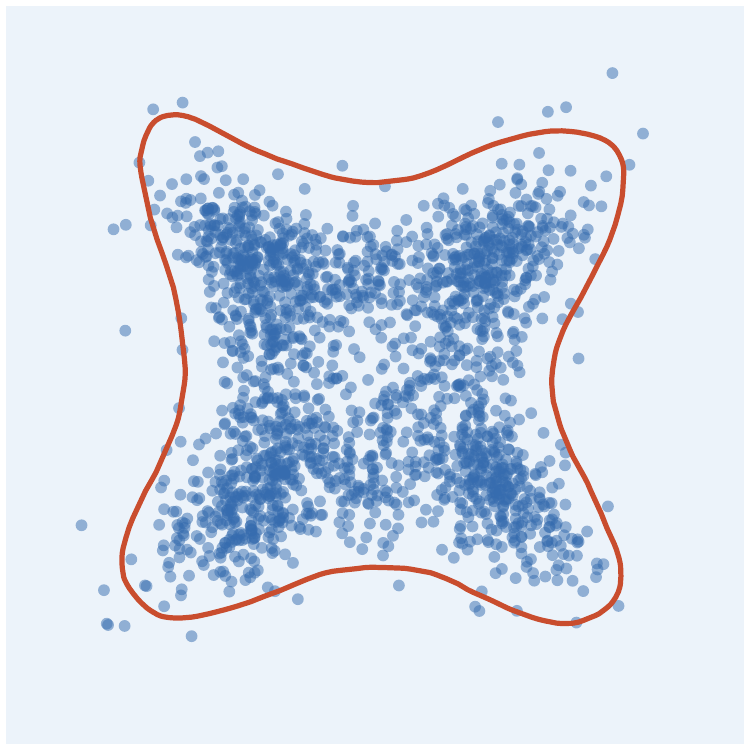}
        \caption{BC.}
    \end{subfigure}\begin{subfigure}{.142\linewidth}
        \centering
        \includegraphics[width=\linewidth]{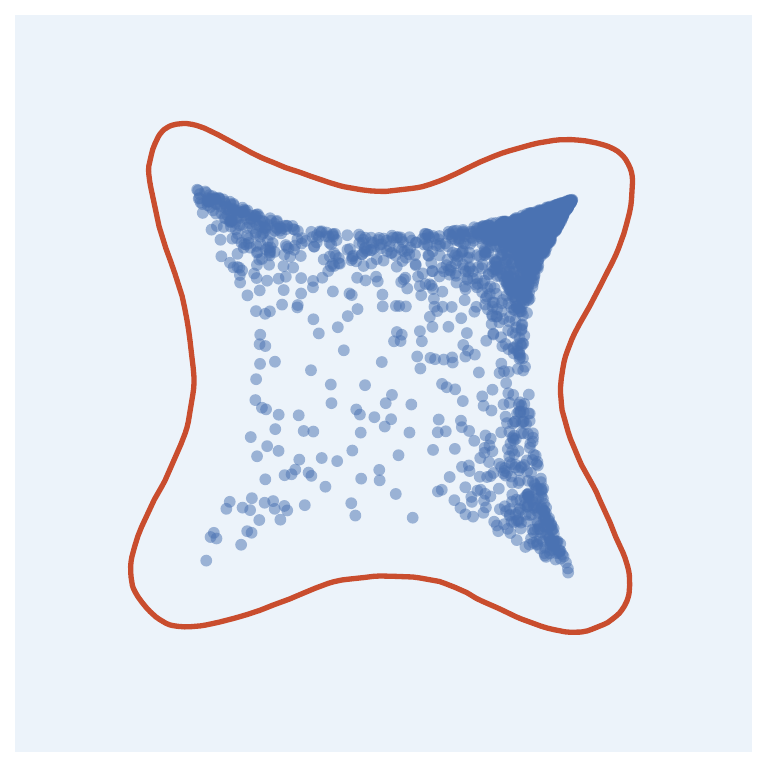}
        \caption{\baselinename{DSRL}.}
    \end{subfigure}\begin{subfigure}{.142\linewidth}
        \centering
        \includegraphics[width=\linewidth]{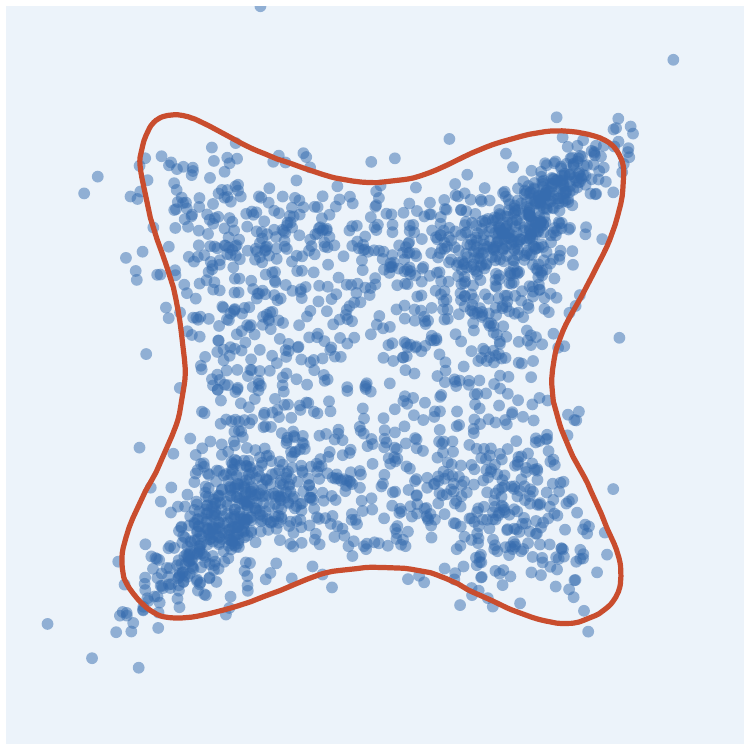}
        \caption{\baselinename{IFQL}.}
    \end{subfigure}\begin{subfigure}{.142\linewidth}
        \centering
        \includegraphics[width=\linewidth]{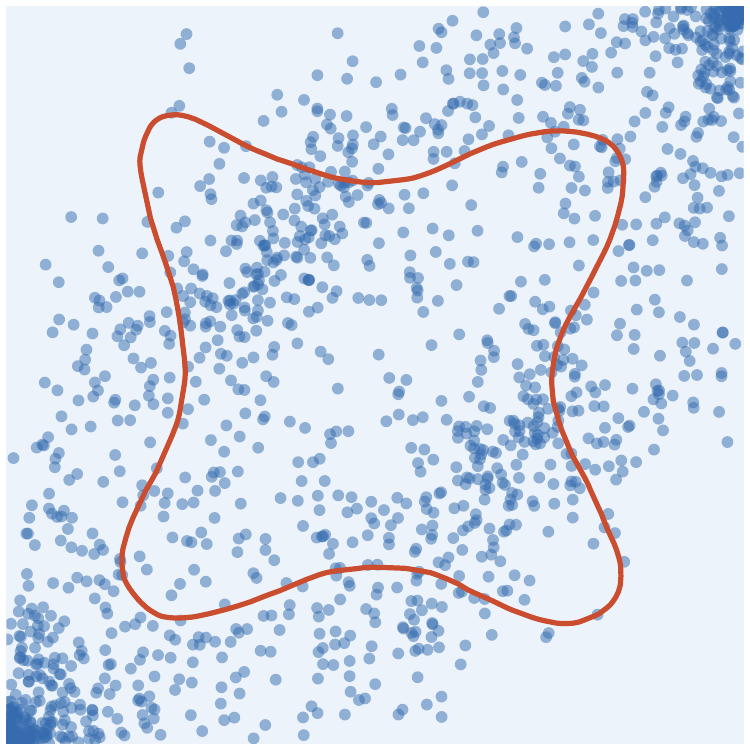}
        \caption{\baselinename{FQL(S)}.}
    \end{subfigure}\begin{subfigure}{.142\linewidth}
        \centering
        \includegraphics[width=\linewidth]{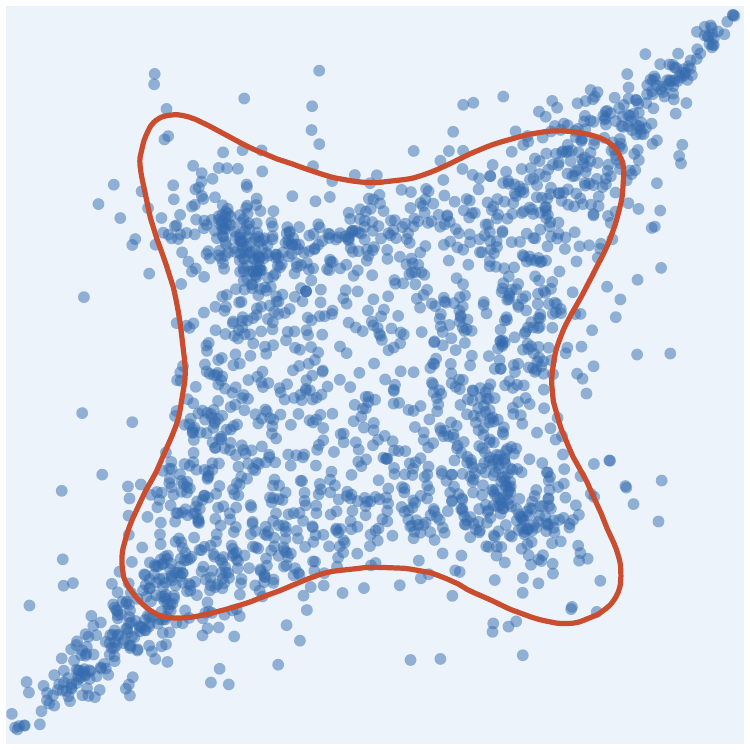}
        \caption{\baselinename{FQL(M)}.}
    \end{subfigure}\begin{subfigure}{.142\linewidth}
        \centering
        \includegraphics[width=\linewidth]{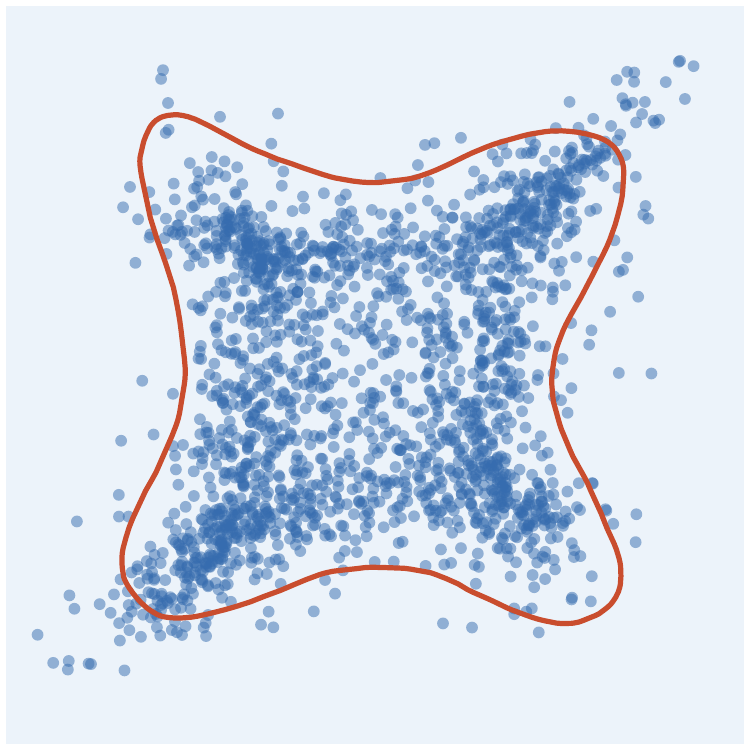}
        \caption{\baselinename{FQL(L)}.}
    \end{subfigure}\begin{subfigure}{.142\linewidth}
        \centering
        \includegraphics[width=\linewidth]{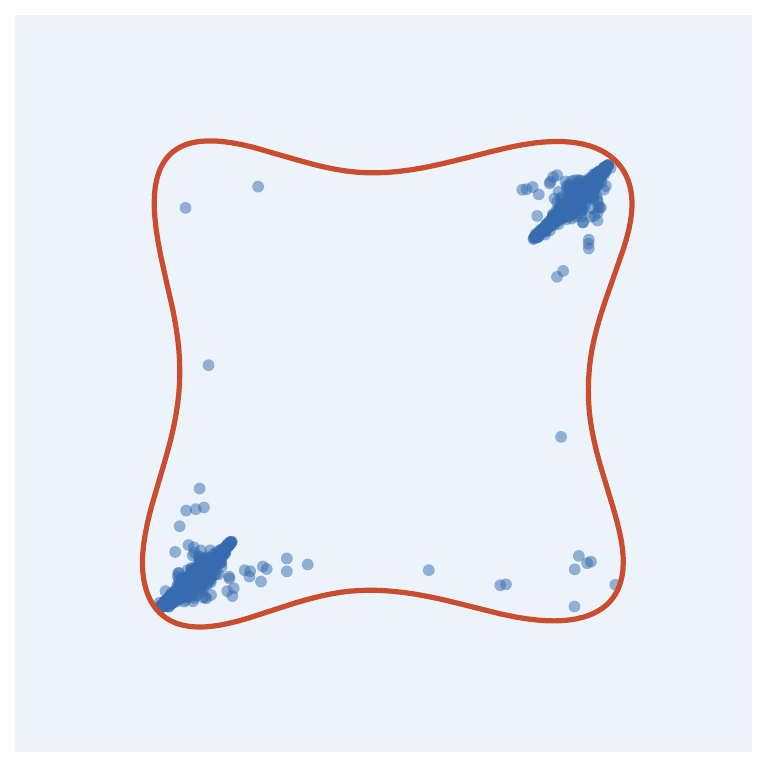}
        \caption{\ourname{}\footnotemark.}
        \setcounter{fnmarkcntr}{\value{footnote}}
    \end{subfigure}\caption{\textbf{Learned policy distributions with the toy example.} The $Q$-value reaches the maximum at the lower left and upper right corners (See the $Q$-value plot in \Cref{fig: algorithm}). The \textcolor{supportred}{red} boundaries denote the estimated $\supp(\pi_\mathrm{BC})$\protect\footnotemark.}
    \label{fig: toy example}
\end{figure}

\setcounter{footnote}{\value{fnmarkcntr}}

\footnotetext{This plotted support slightly differs because $q_\mathrm{BC}=\unif(\gB_l^d)$ for \algo{}, but $q_\mathrm{BC}=\gN(0, I^d)$ for others.}

\addtocounter{footnote}{1}

\footnotetext{The support estimation has some numerical errors, so a few samples of BC/IFQL can be outside.}

\begin{figure}[t]
    \centering
    \includegraphics[height=3.5cm,valign=t]{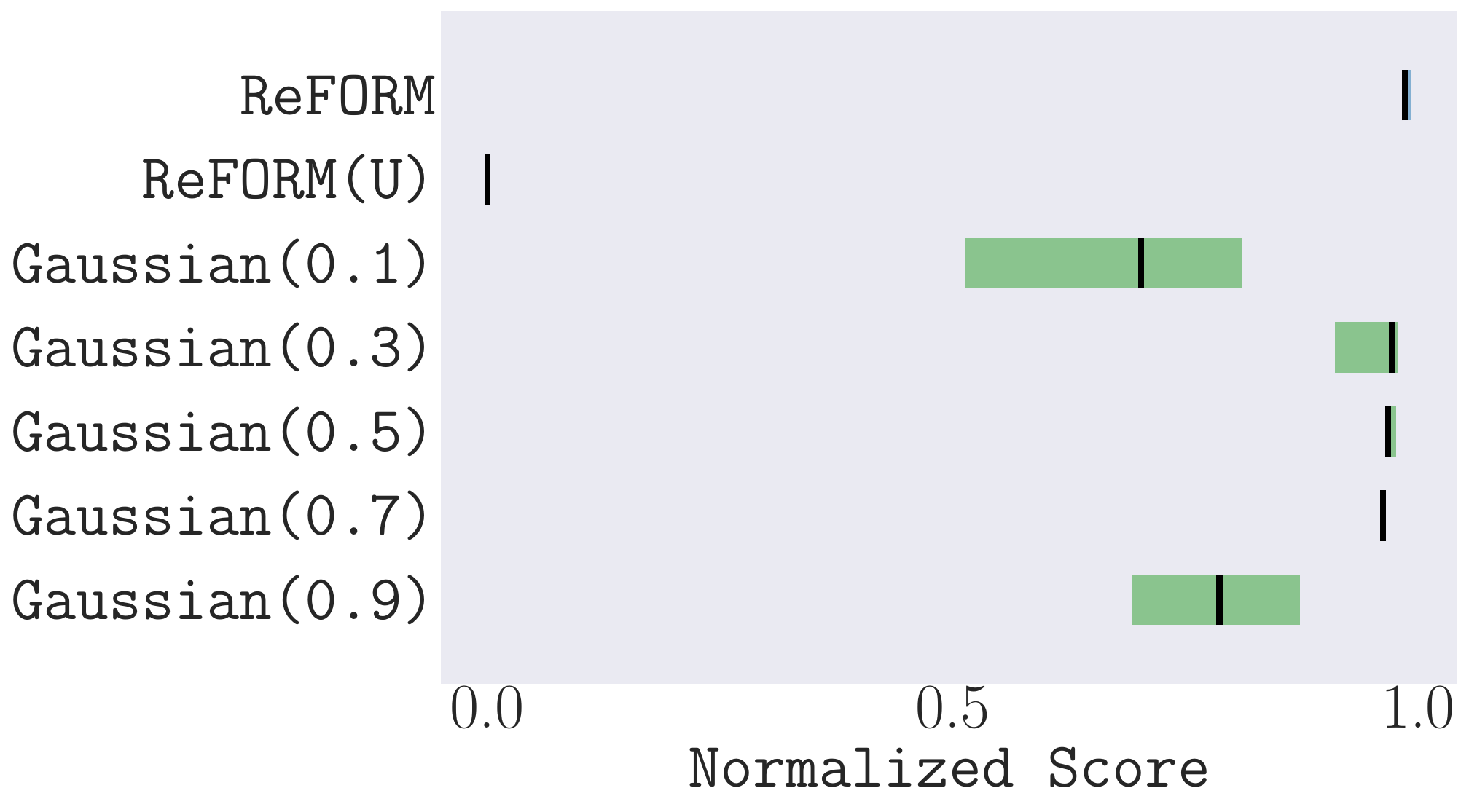}
    \includegraphics[height=3.7cm,valign=t]{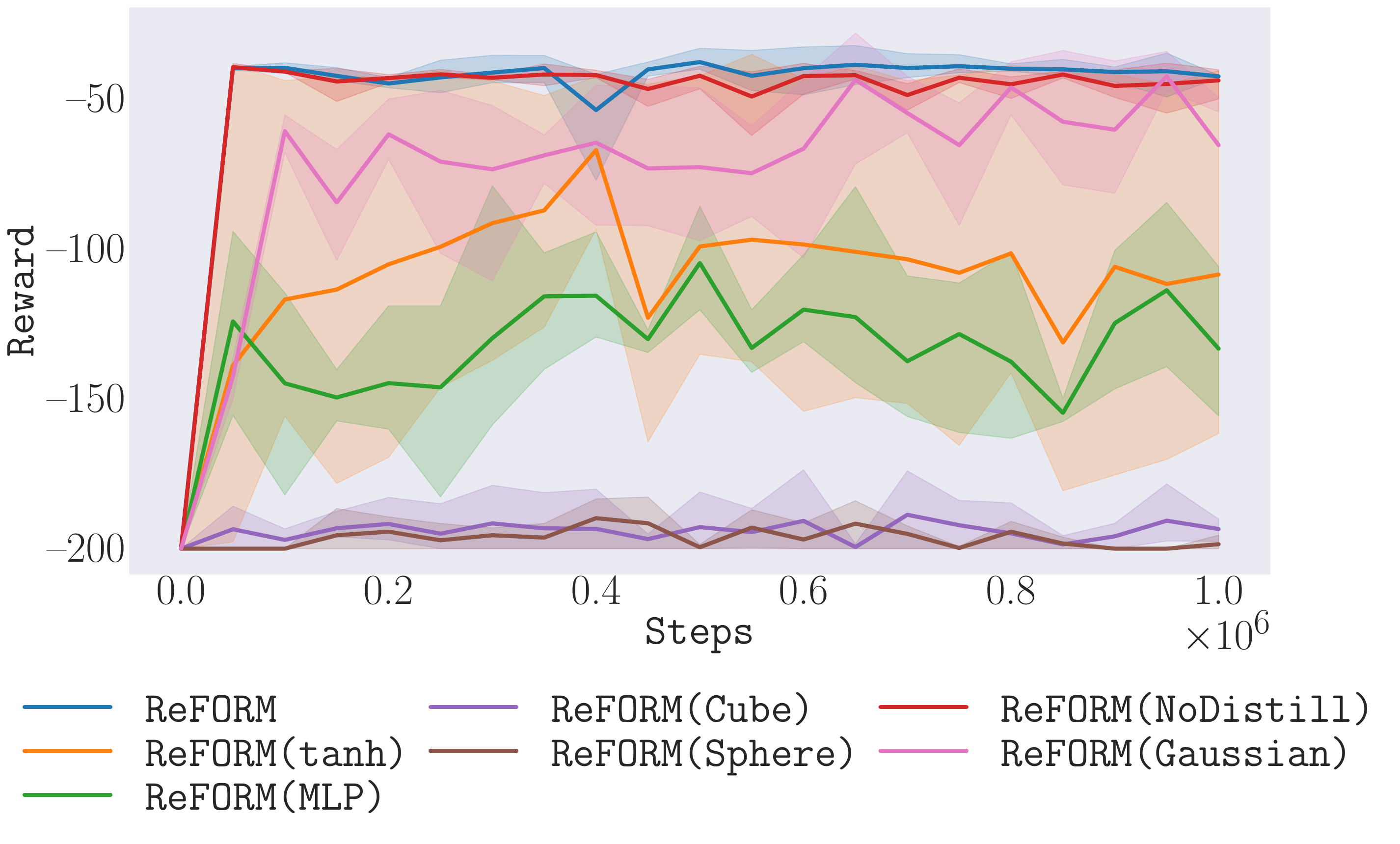}
    \caption{\textbf{Ablations.} 
    Left: normalized scores of \ourname{} and its variants with different source distributions.
Right: training curves of \ourname{} and its variants by changing its components.
}
    \label{fig: ablations}
\end{figure}

To study the functionality of each component of \ourname{}, we conduct the following experiments in a toy environment and the \envname{cube-single} environment with the \datasetname{noisy} dataset to answer Q3-Q5.
All details can be found in~\Cref{app:ablation_details}.

\textbf{(Q3): Having bounded support for the BC flow policy's source distribution is crucial.}
We investigate the effect of satisfying support constraints (and hence the necessity of using a source distribution with bounded support)
by using a Gaussian $\gN(0, I^d)$ with unbounded support as the source distribution for both the BC flow policy and the flow noise generator following \citet{wagenmaker2025steering} (\baselinename{ReFORM(U)}).
\Cref{fig: ablations} shows that \baselinename{ReFORM(U)} suffers from severe OOD problems and does not learn anything.
This confirms that the ability of \algo{} to satisfy support constraints using a source distribution with bounded support is crucial to good performance.

We next change the source distribution of the flow noise generator of \baselinename{ReFORM(U)} back to $\unif(\gB_l^d)$ while keeping $q_\mathrm{BC} = \gN(0, I^d)$ for the BC flow policy. We also add the reflection term back to the noise generator. 
We change $l$ so that $\gB_l^d$ is the $\xi$-confidence level of $q_\mathrm{BC}$.
We vary $\xi$ within $\{0.1, 0.3, 0.5, 0.7, 0.9\}$ (\baselinename{Gaussian($\xi$)}). 
These baselines are highly sensitive to the choice of $\xi$, whereas \ourname{} both avoids this additional hand-tuned hyperparameter $\xi$ and achieves better performance than the best performing \baselinename{Gaussian($\xi$)} (\Cref{fig: ablations}).

\textbf{(Q4): The reflected flow improves the quality of the generated noise.} 
We consider replacing the reflected flow with three different generative models that also generate noise within the hypersphere $\gB^d_l$: a MLP noise generator (\baselinename{ReFORM(MLP)}), a ``squashed flow'' that applies a $\tanh$ at the end (\baselinename{ReFORM(tanh)}), and a squashed Gaussian (\baselinename{ReFORM(Gaussian)}) similar to \baselinename{DSRL}.
All baselines perform worse than \ourname{} (\Cref{fig: ablations}): the MLP and the Gaussian fail to capture multimodal distributions, while $\mathrm{tanh}$ squashing suffers from gradient vanishing. 

\textbf{(Q5): Our design of the reflection term works the best within our considered choices.}
We consider two other options for the reflection term. 
First, \baselinename{ReFORM(Cube)} replaces our hypersphere-shaped domain $\gB_l^d$ with a hypercube-shaped domain, while still applying the reflection term as introduced in \citet{xie2024rfm}. 
Second, \baselinename{ReFORM(Sphere)} shares our hypersphere-shaped domain, but instead of compensating the outbound velocity, it reflects the outbound velocity back inbound once the sample hits $\partial \gB_l^d$.
\Cref{fig: ablations} shows that these two variants cannot perform similarly to \ourname{}. We hypothesize that compensating for the outbound velocities makes the training process more stable than reflecting the outbound velocities. 
We leave finding theoretical explanations of this phenomenon to future work.

\textbf{(Q6): Removing the BC flow policy distillation slightly degrades the performance of \ourname{}.} We compare \ourname{} with its variant \baselinename{ReFORM(NoDistill)} by removing the distillation of the BC flow policy. \Cref{fig: ablations} shows that \baselinename{ReFORM(NoDistill)}'s performance decreases slightly compared with \ourname{}. This suggests that a longer backpropagation chain can be harmful, which matches the observation in \citet{fql_park2025}.

\vspace{-1em}
\section{Conclusion}

We propose \algo{} for realizing the support constraint with flow policies in offline RL.
\algo{} simultaneously learns a BC flow policy that transforms a bounded uniform distribution in a hypersphere to the complex action distribution that matches the behavior policy, and a flow noise generator that transforms a bounded uniform distribution to a complex noise distribution being fed into the BC policy. 
With reflected flow on the noise generator, the noise generator is capable of generating complex multimodal noise while staying within the domain of the prior distribution of the BC policy. Therefore, \algo{} avoids the OOD issues by construction, putting no further constraints limiting the performance of the learned policy, and learns a complex multimodal policy. Our extensive experiments on $40$ challenging tasks with the OGBench offline RL benchmark suggest that \algo{} achieves the best performance with only a single set of hyperparameters, eliminating the costly fine-tuning process of most offline RL methods. The reflected flow noise generator can also be potentially combined with other generative-model-based policies, including diffusion policies.

\paragraph{Limitations.}
We identify several promising avenues for future work. Although our distillation step avoids BPTT through the BC flow, training the noise generator still relies on BPTT, which can be computationally intensive for deep models. This process can be potentially improved with shortcut models \citep{sorl}, or by applying a pre-trained BC model and latent space RL \citep{wagenmaker2025steering}. Furthermore, our method ensures that the policy $\pi_\theta$ remains within the support of the BC policy, meaning that it inherits any potential OOD errors made by the BC model itself. Integrating behavior cloning methods with stricter support constraints, diagnosing when the BC model generates OOD errors, or applying a pre-trained BC model could mitigate this dependence. Moreover, the design of the reflection term is a nascent area, and exploring more adaptive or even learned reflection terms presents an exciting direction for developing more powerful policy improvement methods.
In addition, \algo{} applies the simplest value function learning method and actor-critic structure similar to \citet{fql_park2025}, which can be potentially improved by other methods \citep{mao2023svr,xql,liu2024adaptive,agrawalla2025floq}.
Finally, \algo{} learns slower than algorithms imposing statistical distance regularization when the dataset contains expert policies due to the lack of any explicit regularization to keep the learned policy close to the expert policy. 

\section*{Reproducibility statement}

For better reproducibility, we provide all the proofs of theoretical results in \Cref{app: proof}, and implementation details, including all hyperparameters in each environment of all algorithms in \Cref{app: implementation}. The benchmark we use is open-source and published in \citet{ogbench_park2025}. We also included the source code of our algorithm in the supplementary materials. 

\section*{Acknowledgements}

This material is based upon work supported by the Under Secretary of Defense for Research and Engineering under Air Force Contract No. FA8702-15-D-0001 or FA8702-25-D-B002. In addition, Zhang, So, and Fan are supported by the National Science Foundation (NSF) CAREER Award \#CCF-2238030. This work was also supported in part by a grant from Amazon. Any opinions, findings, conclusions or recommendations expressed in this material are those of the author(s) and do not necessarily reflect the views of the Under Secretary of Defense for Research and Engineering.

© 2025 Massachusetts Institute of Technology.

Delivered to the U.S. Government with Unlimited Rights, as defined in DFARS Part 252.227-7013 or 7014 (Feb 2014). Notwithstanding any copyright notice, U.S. Government rights in this work are defined by DFARS 252.227-7013 or DFARS 252.227-7014 as detailed above. Use of this work other than as specifically authorized by the U.S. Government may violate any copyrights that exist in this work.

\bibliography{camera_ready}
\bibliographystyle{iclr2026_conference}

\newpage
\appendix
\section{Proofs}\label{app: proof}

\subsection{Proof of \Cref{thm: KL implies support}}

\begin{proof}
    We first prove the first statement.   
    We prove this by contradiction. Suppose $\supp(\pi_\theta(\cdot|s))\not\subseteq\supp(\pi_\beta(\cdot|s))$. Then, there exists a region $\gB = \{a\in\gA\mid\pi_\theta(a|s)>0, \pi_\beta(a|s)=0\}$ with a non-zero measure. By the definition of the KL divergence, we have
    \begin{equation}
    \begin{aligned}
        D_\mathrm{KL}(\pi_\theta(\cdot|s)\mid\mid\pi_\beta(\cdot|s)) &= \int_{a\in\gB} \pi_\theta(a|s)\log\frac{\pi_\theta(a|s)}{\pi_\beta(a|s)}da + \int_{a\in\gA\setminus\gB}\pi_\theta(a|s)\log\frac{\pi_\theta(a|s)}{\pi_\beta(a|s)}da,
    \end{aligned}
    \end{equation}
    where the first term is $\infty$ and the second term is finite. Therefore, we have $D_\mathrm{KL}(\pi_\theta(\cdot|s)\mid\mid\pi_\beta(\cdot|s)) = \infty$, which contradicts the condition that $D_\mathrm{KL}(\pi_\theta(\cdot|s)\mid\mid\pi_\beta(\cdot|s))\leq\epsilon<\infty$. 

    We then prove the second statement. 
    Consider $\pi_\theta(\cdot|s) = \gN(\mu, 1)$ and $\pi_\beta(\cdot|s) = \gN(0, 1)$. We have $\supp(\pi_\theta(\cdot|s))\subseteq\supp(\pi_\beta(\cdot|s))$. The KL divergence between them is
    \begin{equation}
        D_\mathrm{KL}(\pi_\theta(\cdot|s)\mid\mid\pi_\beta(\cdot|s)) = \frac{\mu^2}{2}.
    \end{equation}
    Therefore, for any $M > 0$, we can choose $\mu>\sqrt{2M}$ so that $D_\mathrm{KL}(\pi_\theta(\cdot|s)\mid\mid\pi_\beta(\cdot|s))>M$.
\end{proof}

\subsection{Proof of \Cref{thm: W2 neq support}}

\begin{proof}
    For simplicity, consider a given state $s\in\gS$. We define $p_\beta(\cdot) = \pi_\beta(\cdot|s)$ and $p_\theta(\cdot) = \pi_\theta(\cdot|s)$. 
    We prove by construction. 
We consider the optimal transport problem. First, we define a source region within the support of $p_\beta$. Consider a small ball $\gB_1\in\supp(p_\beta)$ centered at $a_1$. The probability mass in the ball is $\delta = \int_{\gB_1} p_\beta(a)da$. Second, we define a target region. Consider another small ball $\gB_2\not\subset\supp(p_\beta)$ centered at $a_2$ with the same radius as $\gB_1$. Let the distance between the two balls be $d = \|a_1-a_2\|$. We define the new probability $p_\theta$ such that 
    \begin{equation}
        p_\theta(a) = 
        \begin{cases} 
        p_\beta(a), & \text{if } a \notin \gB_1 \text{ and } a \notin \gB_2, \\
        0, & \text{if } a \in \gB_1, \\ 
        p_\beta(a - a_2 + a_1), & \text{if } a \in \gB_2 ,
        \end{cases}
    \end{equation}
Then, we have $\supp(p_\theta)\not\subseteq\supp(p_\beta)$. We make $d\leq \sqrt{\frac{\epsilon^2}{\delta}}$ by choosing the source region $\gB_1$ close to the boundary of $\supp(p_\beta)$ and the target region $\gB_2$ close to $\gB_1$. Then, we have
    \begin{equation}
        \begin{aligned}
        D_\mathrm{W2}(p_\theta\mid\mid p_\beta)^2\leq \int_{a\in\gB_1} \|d\|^2p_\beta(a)da = d^2\int_{a\in\gB_1}p_\beta(a)da = d^2\delta \leq \epsilon^2.
        \end{aligned}
    \end{equation}
    Therefore, we have $D_\mathrm{W2}(p_\theta\mid\mid p_\beta)\leq\epsilon$.
\end{proof}

\subsection{Proof of \Cref{thm: reflected_flow}}

\begin{proof}
    Remember that the source distribution of the BC flow policy is $q_\mathrm{BC}=\unif(\gB_l^d)$. 
    We prove the theorem by showing that $z_k\in\unif(\gB_l^d)$ for all $k\in\{0, 1, \dots, N-1\}$, which implies that $z\in\gB_l^d$, for all $z\sim \tilde q_\mathrm{BC}$. We prove this by induction. 

    First, we have $z_0 = w \in \gB_d^l$ because $w\sim\unif(\gB_l^d)$. Next, we assume that $z_k\in\unif(\gB_l^d)$. Then, we have the following two cases:

    \textbf{Case 1:  $\|\hat z_{k+1}\|\le l$.} Following \Cref{eq: euler reflection}, we have $z_{k+1} = \hat z_{k+1}\in\gB_l^d$.

    \textbf{Case 2: $\|\hat z_{k+1}\|> l$.} Following \Cref{eq: euler reflection}, we have
    \begin{equation}
        \begin{aligned}
            z_{k+1} &= \hat z_{k+1} - \langle v_{\theta_2}(k\Delta t,  w; s)\Delta t, n_{k+1} \rangle n_{k+1}  \\
            &= \left(\|\hat z_{k+1}\| - \langle v_{\theta_2}(k\Delta t,  w; s)\Delta t, n_{k+1} \rangle\right)\, n_{k+1}.
        \end{aligned}
    \end{equation}
    In addition, we have
    \begin{equation}
        \langle v_{\theta_2}(k\Delta t,  w; s)\Delta t, n_{k+1} \rangle
        = \langle \hat z_{k+1} - z_k, n_{k+1} \rangle
        = \|\hat z_{k+1}\| - \langle z_k, n_{k+1} \rangle.
    \end{equation}
Plugging this into the previous equation, we get, 
    \begin{equation}
        z_{k+1} = \langle z_k, n_{k+1} \rangle\, n_{k+1}.
    \end{equation}
    Hence, we get,
    \begin{equation}
        \|z_{k+1}\| = \left|\langle z_k, n_{k+1} \rangle\right| \leq \|z_k\| \le l
    \end{equation}
    Thus our reflection ensures $z_{k+1} \in \mathcal{B}^l_d$, $\forall k$. Therefore, we have $z = z_N\in\gB_l^d$, for all $z\sim \tilde q_\mathrm{BC}$. As a result, $\supp(\tilde q_\mathrm{BC})\subseteq\supp(q_\mathrm{BC})$.
\end{proof}

\subsection{Proof of \Cref{thm: support_cnstrnt}}

\begin{proof}
    Let $\tilde z\sim \tilde q_\mathrm{BC}$ be a sample from $\tilde q_\mathrm{BC}$. We have $\tilde z \in \supp(\tilde q_\mathrm{BC})$. Following \Cref{thm: reflected_flow}, we have $\supp(\tilde q_\mathrm{BC})\subseteq\supp(q_\mathrm{BC})$. Therefore, we have $\tilde z \in\supp(q_\mathrm{BC})$. Now consider the original target distribution $p_\mathrm{BC}$. Its support is the set of all points generated by applying the flow $\psi_{\theta_1}$ to all points in the support of $q_\mathrm{BC}$, i.e., 
    \begin{equation}
        \supp(p_\mathrm{BC}) = \{\psi_{\theta_1}(1, z; s)\mid z\in\supp(q_\mathrm{BC})\}.
    \end{equation}
    Since we have $\tilde z \in\supp(q_\mathrm{BC})$, then by definition, we have $\psi_{\theta_1}(1, \tilde z; s) \in \supp(p_\mathrm{BC})$. This is true for all $\tilde z\sim \tilde q_\mathrm{BC}$. Therefore, by the definition of the support of $\tilde p_\mathrm{BC}$, i.e.,
    \begin{equation}
        \supp(\tilde p_\mathrm{BC}) = \{\psi_{\theta_1}(1, \tilde z; s)\mid \tilde z\in\supp(\tilde q_\mathrm{BC})\},
    \end{equation}
    we have $\supp(\tilde p_\mathrm{BC})\subseteq \supp(p_\mathrm{BC})$.
\end{proof}

\section{Algorithm Details}

We provide the step-by-step explanation of \algo{} in \Cref{alg: reform}, where $\mathrm{RF}(v, s, w, N)$ means solving the reflected ODE \eqref{eq: reflected ode} following the projected Euler step \eqref{eq: euler reflection} with the velocity field $v$, state $s$, sample from the source distribution $w$, and number of Euler steps $N$.

\begin{algorithm}[t]
\caption{ReFORM Algorithm}
\label{alg: reform}
\begin{algorithmic}[1]
\State \textbf{Input:} Offline dataset $\mathcal{D}$;  total Euler number of steps $N$, radius $l$ 
\State \textbf{Networks:} Critic $Q_\phi(\pl s,\pl a)$; BC flow field $v_{\theta_1}(\pl t, \pl z; \pl s)$; noise flow field $v_{\theta_2}(\pl t, \pl w; \pl s)$; one-step BC flow policy $\mu_{\hat\theta_1}(\pl z; \pl s)$.

\While{not converged}
    \State Sample batch $\{(s, a, r, s')\} \sim \gD$
    \State 
    \State {\bfseries $\triangleright$ Critic update}
    \State $w\sim\unif(\gB_l^d)$
    \State $z\gets\mathrm{RF}(v_{\theta_2}, s', w, N)$
    \State $a'\gets\mu_{\hat\theta_1}(z; s')$
    \State Update $\phi$ to minimize $\mathbb{E}\left[(r+\gamma Q_{\hat\phi}(s',a') - Q_\phi(s, a))^2)\right]$
    \State
    \State {\bfseries \(\triangleright\) Train vector field \(v_{\theta_1}\) in the BC flow policy \(\mu_{\theta_1}\)}
    \State \(z \sim \unif(\gB_l^d)\)
    \State \(x_1 \gets a\)
    \State \(t \sim \unif[0,1]\)
    \State \(x_t \gets (1-t)\,z + t\,x_1\)
    \State Update \(\theta_1\) to minimize
    $\mathbb{E}\left[\|v_{\theta_1}(t, x_t;s)-(x_1 - z)\|^2\right]$
    \State
    \State {\bfseries $\triangleright$ Train one-step policy $\mu_{\hat\theta_1}$}
    \State $z \sim \unif(\gB_l^d)$
    \State $a^{\mu_1} \gets \mu_{\hat\theta_1}(z;s)$
    \State Update $\hat\theta_1$ to minimize $\mathbb{E}\left[\|a^{\mu_1} - \mu_{\theta_1}(z; s)\|^2\right]$
    \State
    \State {\bfseries $\triangleright$ Train vector field $v_{\theta_2}$ in the flow noise generator $\mu_{\theta_2}$}
    \State $w \sim \unif(\gB_l^d)$
    \State $z \gets \mathrm{RF}(v_{\theta_2}, s, w, N)$
    \State $a^{\mu_2}\gets\mu_{\theta_1}(z;s)$
    \State Update $\theta_2$ to minimize $\mathbb{E}\left[-Q_\phi(s, a^{\mu_2})\right]$
\EndWhile
\end{algorithmic}
\end{algorithm}

\section{Experiments}\label{app: experiments}

\subsection{Computation resources}

The experiments are run on a 13th Gen Intel(R) Core(TM) i7-13700KF CPU with 64GB RAM and an NVIDIA GeForce RTX 4090 GPU. The training time is around $80$ minutes for $10^6$ steps for \algo{}. 

\subsection{Environments}

We conduct experiments on the recently published OGBench benchmark \citep{ogbench_park2025}. We use $4$ environments ($1$ locomotion environment and $3$ manipulation environments), $5$ tasks in each environment, with $2$ different datasets, for a total $40$ tasks. Since OGBench was originally designed for offline goal-conditioned RL, we use the single-task variants ("\envname{-singletask}") for OGBench tasks to benchmark standard reward-maximizing offline RL. The reward functions in OGBench are semi-sparse. For the locomotion task, the reward functions are always $-1$ for not reaching the goal and $0$ for reaching the goal. Manipulation tasks usually contain several subtasks, and the rewards are bounded by $-n_\mathrm{task}$ and $0$, where $n_\mathrm{task}$ is the number of subtasks. All episodes end when the agent achieves the goal.

In our experiments, we consider the following tasks with the \datasetname{clean} dataset, where the demonstrations are randomly generated by an expert policy:
\begin{itemize}
    \item \envname{antmaze-large-navigate-singletask-task\{1,2,3,4,5\}-v0}
    \item \envname{cube-single-play-singletask-task\{1,2,3,4,5\}-v0}
    \item \envname{cube-double-play-singletask-task\{1,2,3,4,5\}-v0}
    \item \envname{scene-play-singletask-task\{1,2,3,4,5\}-v0}
\end{itemize}
We also consider the \datasetname{noisy} dataset, where the demonstrations are randomly generated by a highly suboptimal and noisy policy:
\begin{itemize}
    \item \envname{antmaze-large-explore-singletask-task\{1,2,3,4,5\}-v0}
    \item \envname{cube-single-noisy-singletask-task\{1,2,3,4,5\}-v0}
    \item \envname{cube-double-noisy-singletask-task\{1,2,3,4,5\}-v0}
    \item \envname{scene-noisy-singletask-task\{1,2,3,4,5\}-v0}
\end{itemize}
More details about the environment and videos of the demonstrations can be found in the OGBench paper \citep{ogbench_park2025}.

\subsection{Implementation details and hyperparameters}\label{app: implementation}

\begin{table}[t]
    \centering
    \caption{Training steps for all algorithms for each task.}
    \label{tab: training steps}
    \begin{tabular}{l|cc}
        \toprule
        Task & Dataset & Training step \\
        \midrule
        \envname{antmaze-large-navigate-singletask-task\{1,2,3,4,5\}-v0} & \datasetname{clean} & $1\times10^7$ \\
        \envname{antmaze-large-explore-singletask-task\{1,2,3,4,5\}-v0} & \datasetname{noisy} & $8\times10^6$ \\
        \envname{cube-single-play-singletask-task\{1,2,3,4,5\}-v0} & \datasetname{clean} & $2\times10^6$ \\
        \envname{cube-single-noisy-singletask-task\{1,2,3,4,5\}-v0} & \datasetname{noisy} & $3\times10^6$ \\
        \envname{cube-double-play-singletask-task\{1,2,3,4,5\}-v0} & \datasetname{clean} & $2\times10^6$ \\
        \envname{cube-double-noisy-singletask-task\{1,2,3,4,5\}-v0} & \datasetname{noisy} & $1\times10^6$ \\
        \envname{scene-play-singletask-task1-v0} & \datasetname{clean} & $2\times10^6$ \\
        \envname{scene-play-singletask-task\{2,3,4,5\}-v0} & \datasetname{clean} & $3\times10^6$ \\
        \envname{scene-noisy-singletask-task\{1,2\}-v0} & \datasetname{noisy} & $1\times10^6$ \\
        \envname{scene-noisy-singletask-task\{3,4,5\}-v0} & \datasetname{noisy} & $2\times10^6$ \\
        \bottomrule
    \end{tabular}
\end{table}

\subsubsection{Details of \algo{}}

\paragraph{Flow policies.}
We parameterize the velocity fields of the BC flow policy $v_{\theta_1}$ and the flow noise generator $v_{\theta_2}$ with MLPs. We use the Euler method to solve ODE \eqref{eq: flow} for the BC flow policy, and the projected Euler step \eqref{eq: euler reflection} to solve the reflected ODE \eqref{eq: reflected ode} for the flow noise generator. $10$ Euler steps are used for both Euler integration for all environments.

\paragraph{$Q$-functions.}
Following the standard implementation of $Q$-functions in RL, we train two $Q$ functions to improve stability. Two aggregation methods are used to aggregate the two $Q$-values for different environments following \citet{fql_park2025}. For most environments, we take the mean of the two $Q$-values for aggregation \citep{ball2023efficient,nauman2024bigger}, except for the \envname{antmaze-large} environment, where we take the minimum of the two $Q$-values \citep{van2016deep,fujimoto2018addressing}.

\paragraph{Selection of the radius of the hypersphere $\gB_l^d$.}
As the action space for physical systems is always compact, we select the hypersphere $\gB_l^d$ to be the smallest hypersphere that contains the action space, i.e., $l = \min_{l'}\{l'\in\mathbb{R}^d \mid \mathcal{A} \subseteq \gB_l^d \}$. Note that, as the action space $\mathcal{A}$ is known and is usually a hyperbox, in most cases, we can compute the solution easily, or, otherwise, use an overapproximation of $\gB_{l}^d$. Therefore, this choice does not impose any limitation on our approach. We also present experimental results of the sensitivity of \algo{} w.r.t. $l$ in \Cref{app: additional exp results}.

\paragraph{Neural Network architectures.}
For all neural networks in our experiments, we use MLPs with $4$ hidden layers and $512$ neurons on each layer. We apply layer normalization \citep{ba2016layer} to the $Q$-function networks to stabilize training. 

\paragraph{Training and evaluation.}
The difficulty of tasks in OGBench can be very different. Therefore, we use different training steps for different tasks (\Cref{tab: training steps}). For each task, we train each algorithm with $3$ different seeds and evaluate the model saved at the last epoch for $32$ episodes. 

\begin{table}[t]
    \centering
    \caption{Common hyperparameters for all algorithms.}
    \label{tab: common hyperparams}
    \begin{tabular}{l|l}
        \toprule
        Hyperparameter & Value \\
        \midrule
        Learning rate & $0.0003$ \\
        Optimizer & Adam \citep{adam} \\
        Maximum gradient norm & $10$ \\
        Target network smoothing coefficient & $0.005$ \\
        Discount factor $\gamma$ & $0.995$ \\
        MLP dimensions & $[512, 512, 512, 512]$ \\
        Nonlinearity & GELU \citep{hendrycks2016gelu} \\
        Flow steps & $10$ \\
        Flow time sampling distribution & $\unif[0, 1]$ \\
        Minibatch size & $256$ \\
        Clipped double $Q$-learning & False (default), True (\envname{antmaze-large}) \\
        \bottomrule
    \end{tabular}
\end{table}

\begin{table}[t]
    \caption{Environment-specific hyperparameters for \baselinename{FQL} and \baselinename{DSRL}.}
    \label{tab: hyperparams fql dsrl}
    \centering
    \begin{tabular}{l|cccc}
        \toprule
        Environment & \baselinename{FQL(S)}$\alpha$ & \baselinename{FQL(M)}$\alpha$ & \baselinename{FQL(L)}$\alpha$ & Noise bound for \baselinename{DSRL} \\
        \midrule
        \envname{antmaze-large} & $1$ & $10$ & $100$ & $[-1.25, 1.25]$ \\
        \envname{cube-single} & $30$ & $300$ & $3000$ & $[-0.5, 0.5]$ \\
        \envname{cube-double} & $30$ & $300$ & $3000$ & $[-1.5, 1.5]$ \\
        \envname{scene} & $30$ & $300$ & $3000$ & $[-0.75, 0.75]$ \\
        \bottomrule
    \end{tabular}
\end{table}

\subsubsection{Details of baselines in main results} 

We choose the state-of-the-art offline RL methods with flow policies as our baselines, including \baselinename{FQL} \citep{fql_park2025}, \baselinename{IFQL} \citep{hansenestruch2023idql,fql_park2025}, and \baselinename{DSRL} \citep{wagenmaker2025steering}. We implement the baselines \baselinename{FQL} and \baselinename{IFQL} following the original implementation provided in \citet{fql_park2025}, and \baselinename{DSRL} also following the original implementation provided in \citet{wagenmaker2025steering}.

\subsubsection{Details of baselines in ablation studies}\label{app:ablation_details}

\paragraph{\baselinename{ReFORM(U)}.}
\baselinename{ReFORM(U)} modifies \ourname{} by changing the source distribution of both the BC policy and the noise generator from $\unif(\gB_l^d)$ to $\gN(0, I^d)$. In other words, we have $q_\mathrm{NG} = q_\mathrm{BC} = \gN(0, I^d)$ for \baselinename{ReFORM(U)}.

\paragraph{\baselinename{Gaussian($\xi$)}.}
\baselinename{Gaussian($\xi$)} modifies \ourname{} by changing the source distribution of the BC policy from $\unif(\gB_l^d)$ to $\gN(0, I^d)$, then choose $l$ so that $\gB_l^d$ is the $\xi$-confidence level of $\gN(0, I^d)$, i.e., $l = \sqrt{\mathrm{PPF}_{\chi_d^2}(\xi)}$, where $\mathrm{PPF}_{\chi_d^2}$ is the percent point function of a $d$-dimensional $\chi^2$ distribution. 

\paragraph{\baselinename{ReFORM(MLP)}.}
\baselinename{ReFORM(MLP)} modifies \ourname{} by changing the reflected flow noise generator to an MLP noise generator $f(\pl s): \gS\to\gB_l^d$, which maps the state to a point within $\supp(q_\mathrm{BC})$. 

\paragraph{\baselinename{ReFORM($\mathrm{tanh}$)}.}
\baselinename{ReFORM($\mathrm{tanh}$)} modifies \ourname{} by removing the reflection term in the reflection ODE \eqref{eq: reflected ode}, i.e., using \eqref{eq: euler} instead of \eqref{eq: euler reflection} when integrating the noise generator flow. Then, after the Euler integration and getting $\hat z$ following \eqref{eq: euler}, we use $\mathrm{tanh}$ to squash the norm of $z$ so that it stays within $\gB_l^d$. In other words, $z= \frac{\hat z}{\|\hat z\|}\cdot\mathrm{tanh}(\|\hat z\|)\cdot l$.

\paragraph{\baselinename{ReFORM(cube)}.}
\baselinename{ReFORM(cube)} modifies \ourname{} by changing the domain of $q_\mathrm{NG}$ and $q_\mathrm{BC}$ to $[-1, 1]^d$. Then, the reflected ODE is solved by first using the Euler integration \eqref{eq: euler} to get $\hat z$, and then applying $z = 1 - |(\hat z + 1)\mathrm{mod} 4 - 2|$ following \citet{xie2024rfm}.

\paragraph{\baselinename{ReFORM(sphere)}.}
\baselinename{ReFORM(sphere)} modifies \ourname{} by changing the reflection term from compensating the outbound velocity to ``bouncing back'', like billiards. 

\paragraph{\baselinename{ReFORM(NoDistill)}.}
\baselinename{ReFORM(NoDistill)} removes the distillation part of \ourname{}, i.e., the actor loss \eqref{eq: actor loss} is backpropagated through the BC flow policy instead of the one-step policy. 

\subsubsection{Hyperparameters}

The choice of hyperparameters largely follows \citet{fql_park2025}. We provide the common hyperparameters shared for all algorithms in \Cref{tab: common hyperparams}, and the environment-specific hyperparameters for \baselinename{FQL} and \baselinename{DSRL} in \Cref{tab: hyperparams fql dsrl}. Note that all environment-specific hyperparameters for \baselinename{FQL(M)} and \baselinename{DSRL} are the same as provided in their original papers (with the \datasetname{clean} dataset), which are hand-tuned for each environment. As the baselines were not tested on the \datasetname{noisy} dataset in their original papers, we use the same hyperparameters for them in the same environment with the \datasetname{clean} dataset.

\begin{figure}[t]
    \centering
    \includegraphics[width=0.24\linewidth]{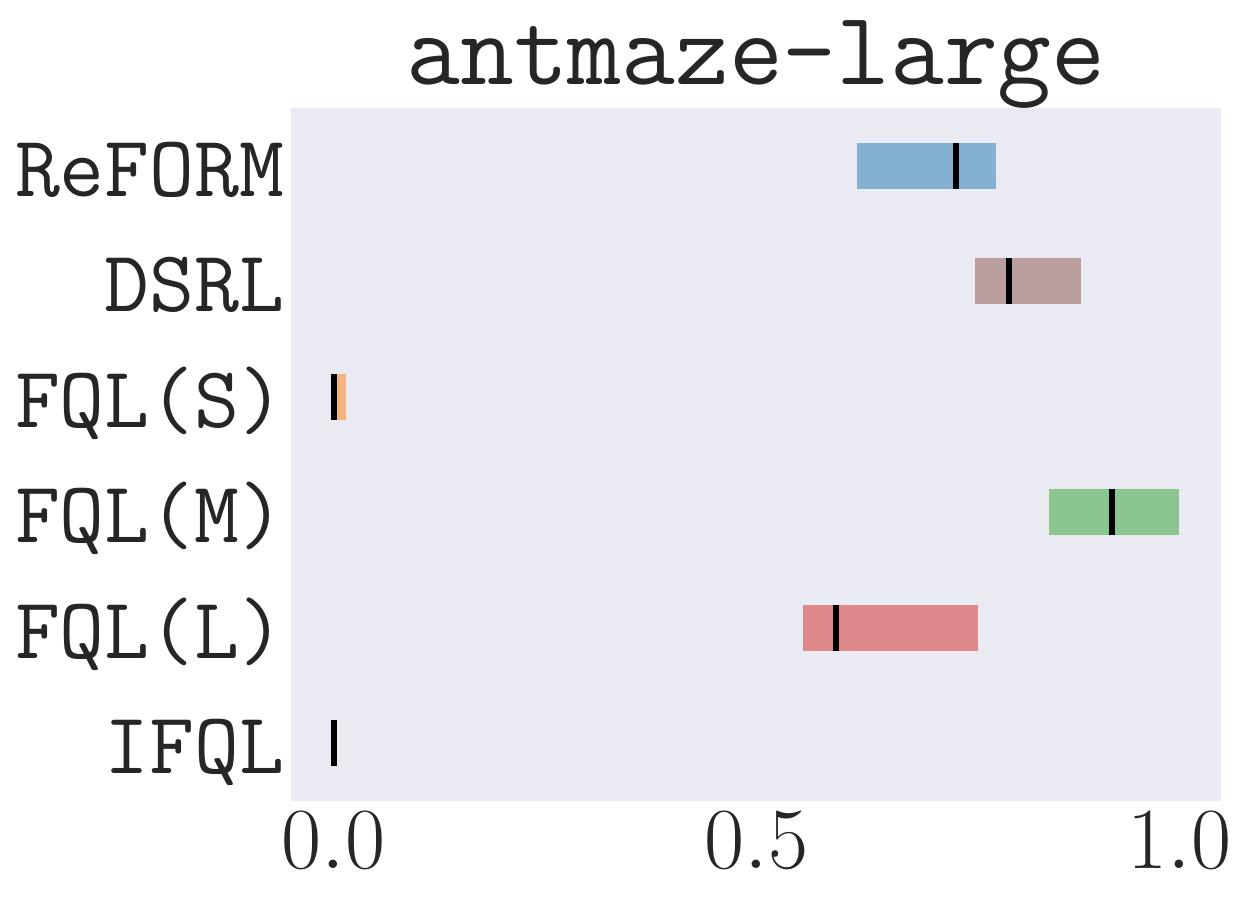}
    \includegraphics[width=0.24\linewidth]{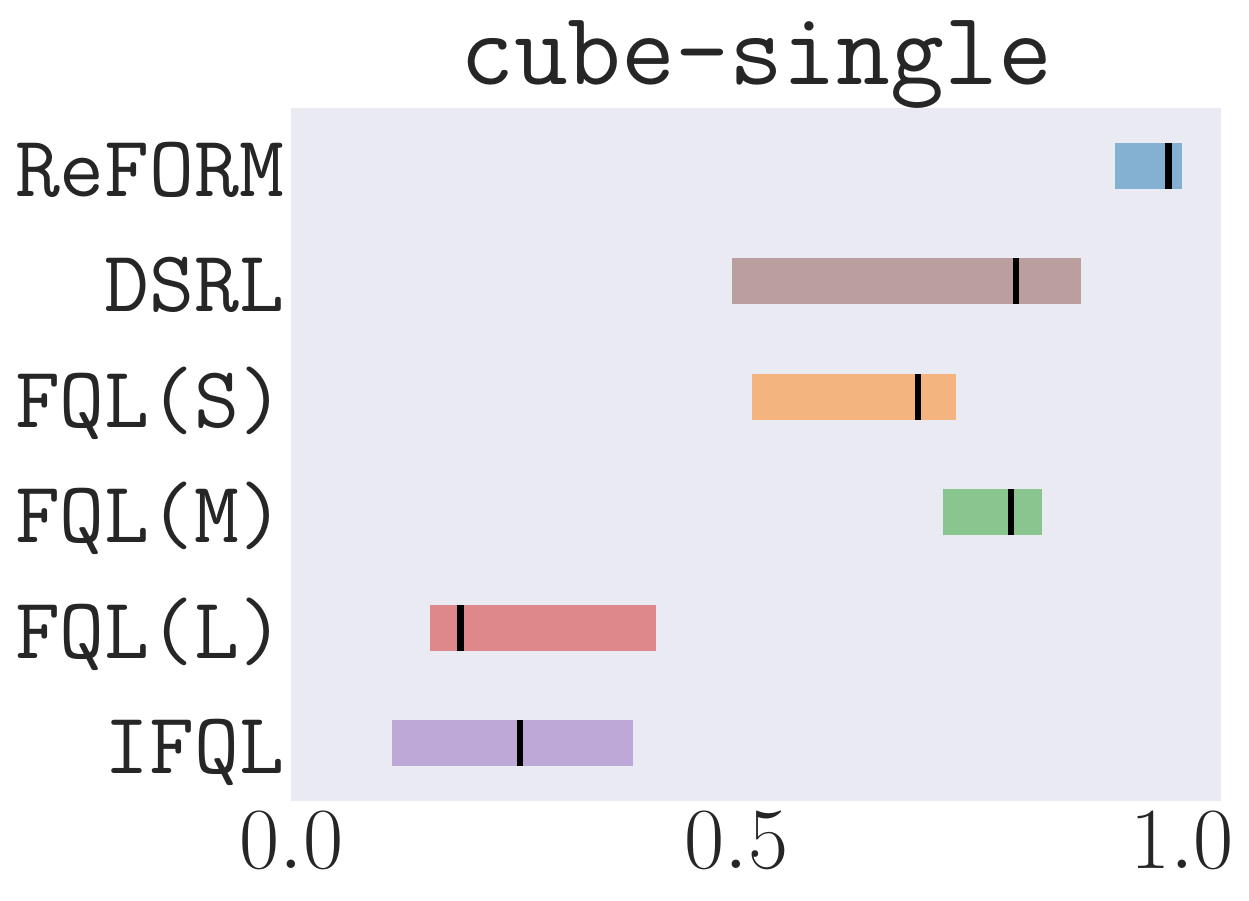}
    \includegraphics[width=0.24\linewidth]{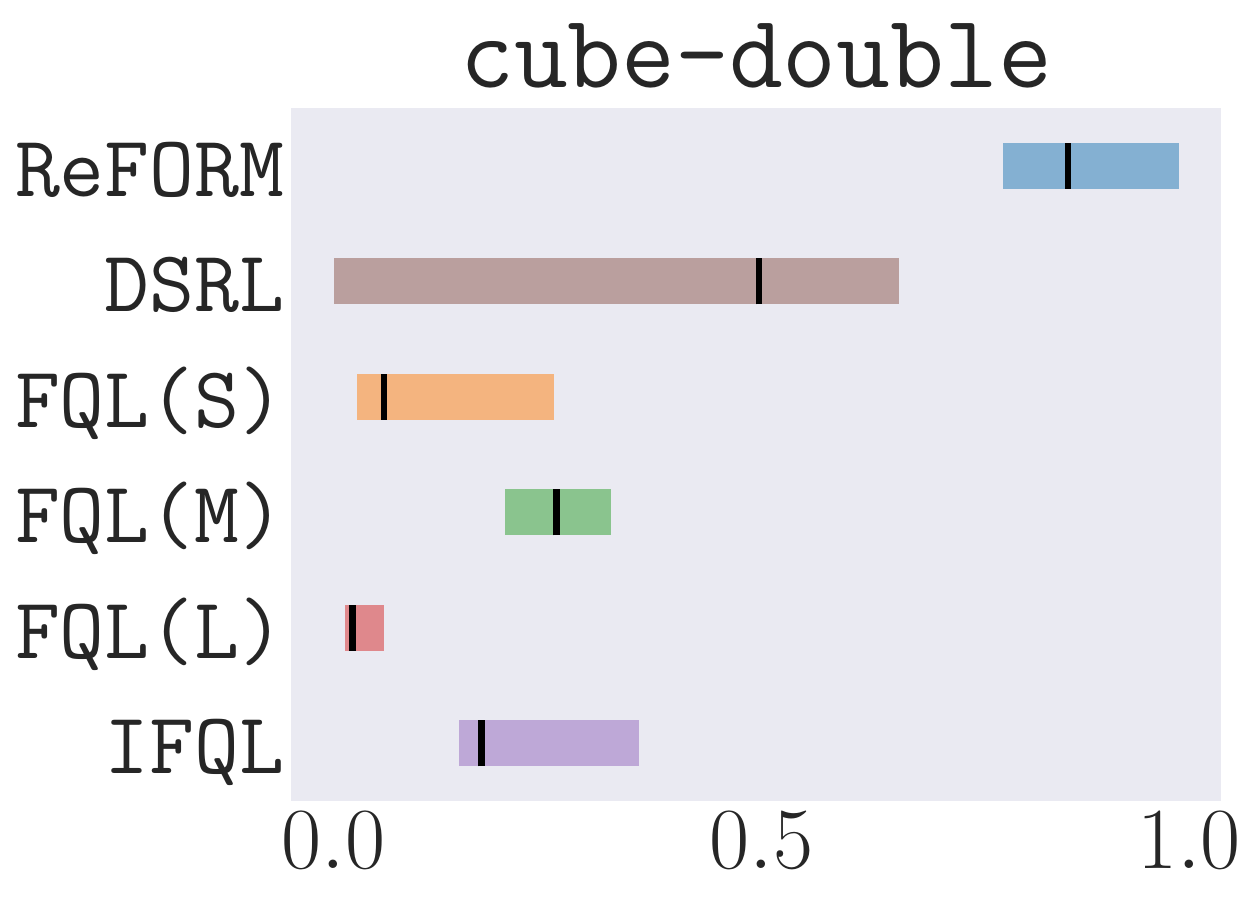}
    \includegraphics[width=0.24\linewidth]{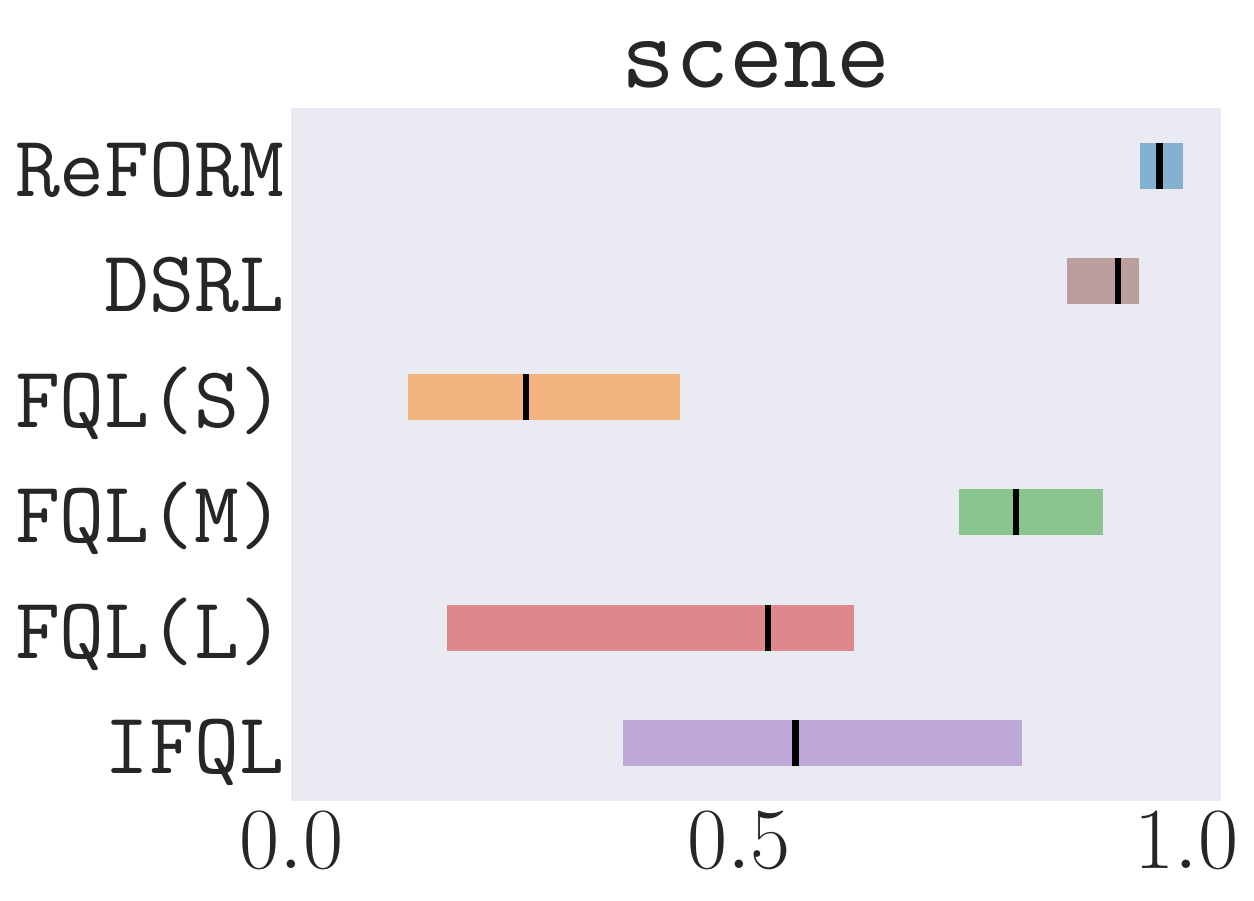}
    \caption{{Normalized scores with the \textsc{clean} dataset.}}
    \label{fig: iqm-clean}
\end{figure}

\begin{figure}[t]
    \centering
    \includegraphics[width=0.24\linewidth]{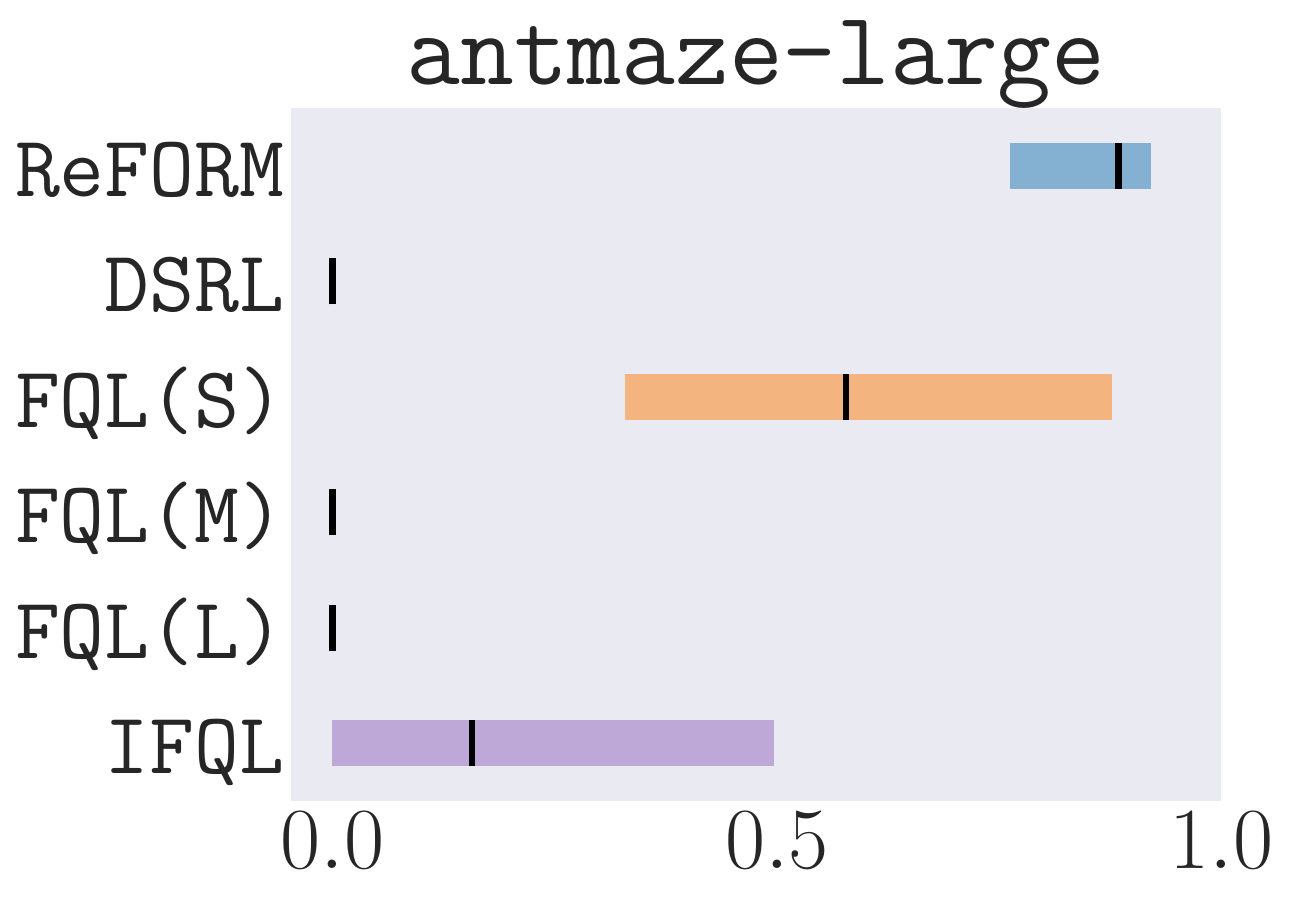}
    \includegraphics[width=0.24\linewidth]{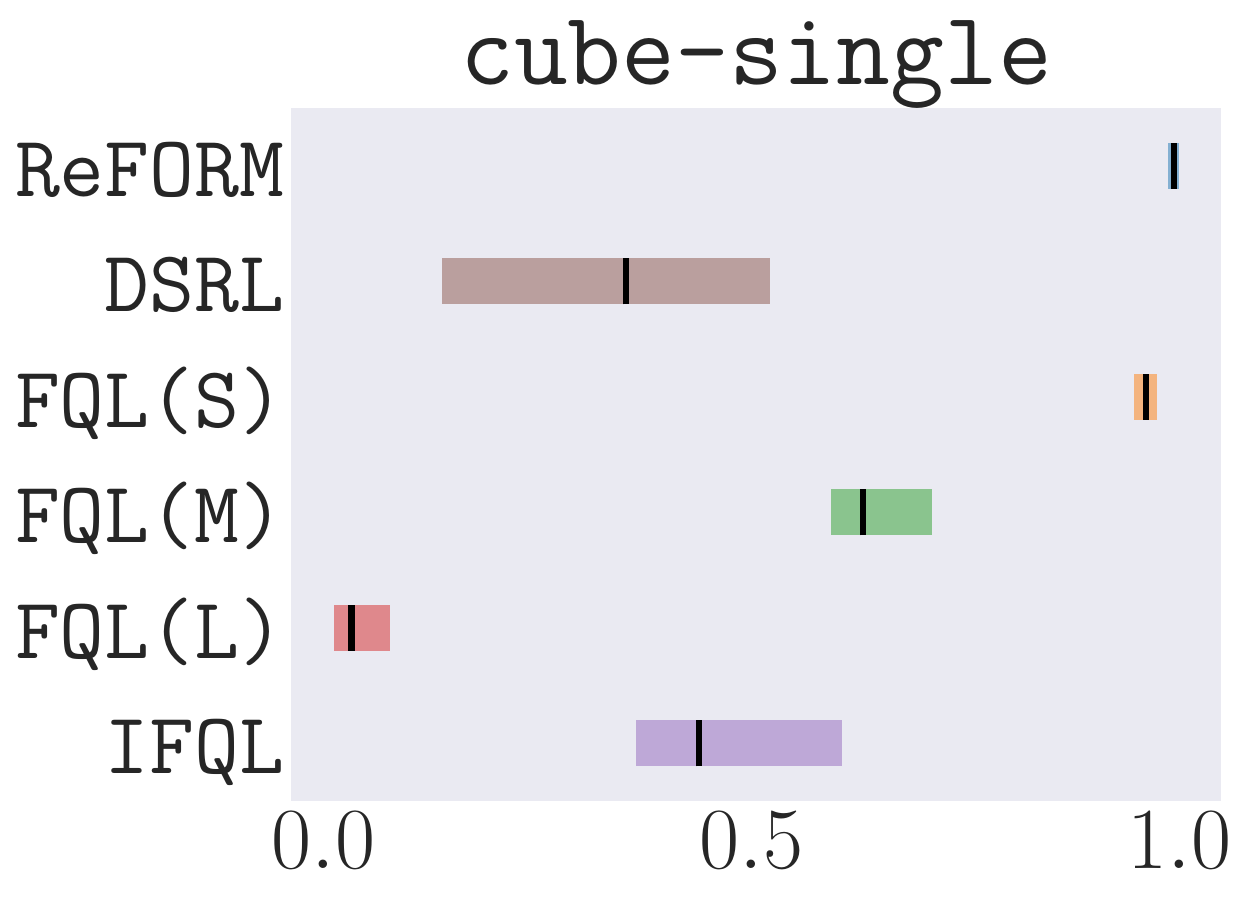}
    \includegraphics[width=0.24\linewidth]{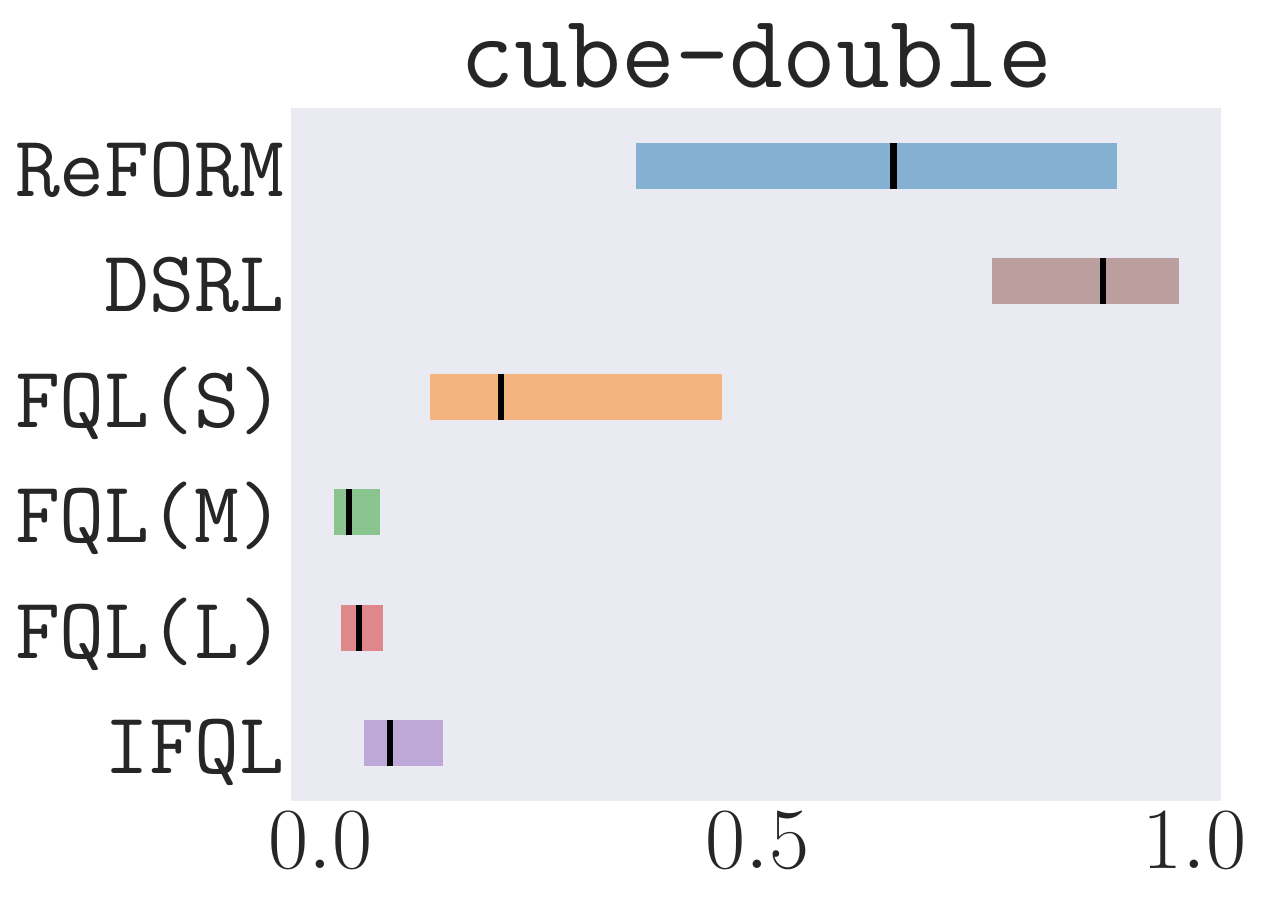}
    \includegraphics[width=0.24\linewidth]{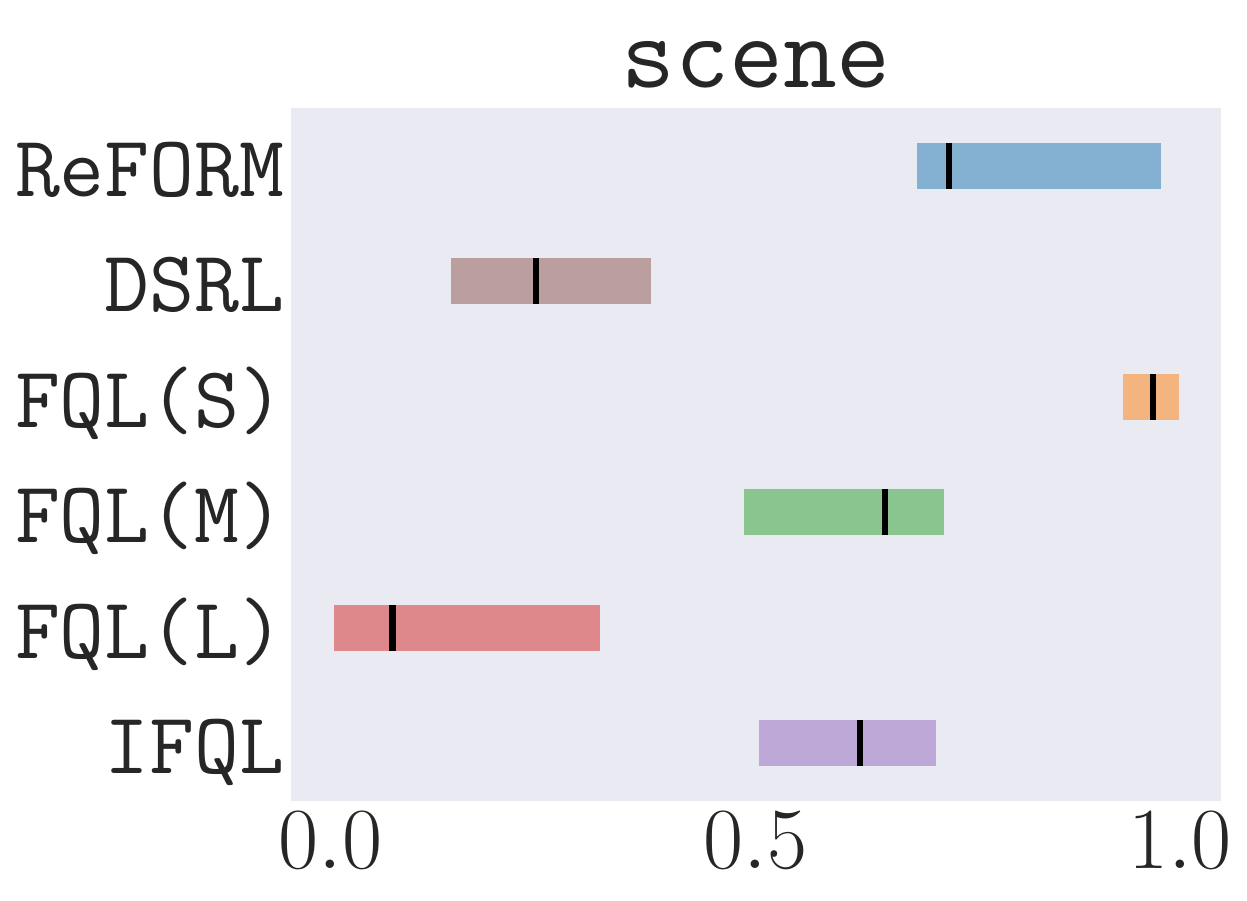}
    \caption{{Normalized scores with the \textsc{noisy} dataset.}}
    \label{fig: iqm-noisy}
\end{figure}

\subsection{Additional results}\label{app: additional exp results}

\paragraph{Normalized scores for each environment and each dataset.}
We present bar plots of the interquantile means (IQM) (see \citet{agarwal2021deep} for details) of the normalized scores for each algorithm in each environment with the \datasetname{clean} dataset (\Cref{fig: iqm-clean}) and the \datasetname{noisy} dataset (\Cref{fig: iqm-noisy}). We can observe that \ourname{} consistently achieves the best or comparable results in all environments with both datasets, with a constant set of hyperparameters. \baselinename{DSRL} and \baselinename{FQL(M)} generally perform the second and third best in environments with the \datasetname{clean} dataset. However, their performance drops when the \datasetname{noisy} dataset is used.

\begin{table}[t]
    \centering
    \caption{\textbf{Full results.} We present full results (normalized score) on $40$ OGBench tasks. The results are averaged over $3$ seeds and $32$ runs per seed. The results are bolded if the algorithm achieves at or above $95\%$ of the best performance following \citet{ogbench_park2025}. To save space, the \envname{-singletask} tags are omitted from task names.}
    \scriptsize
    \label{tab: full result}
	\begin{tabular}{l|l|ccccc|c}
        \toprule
        Task & Dataset & \baselinename{IFQL} & \baselinename{FQL(L)} & \baselinename{FQL(M)} & \baselinename{FQL(S)} & \baselinename{DSRL} & \ourname{} \\
        \midrule
        \envname{antmaze-large-navigate-task1-v0} & \datasetname{clean} & $0${\tiny $\pm0$} & $51${\tiny $\pm2$} & $\mathbf{96}${\tiny $\pm3$} & $1${\tiny $\pm1$} & $85${\tiny $\pm9$} & $65${\tiny $\pm9$} \\
        \envname{antmaze-large-navigate-task2-v0} & \datasetname{clean} & $0${\tiny $\pm0$} & $56${\tiny $\pm4$} & $62${\tiny $\pm44$} & $0${\tiny $\pm0$} & $\mathbf{83}${\tiny $\pm9$} & $72${\tiny $\pm5$} \\
        \envname{antmaze-large-navigate-task3-v0} & \datasetname{clean} & $0${\tiny $\pm0$} & $\mathbf{91}${\tiny $\pm6$} & $\mathbf{90}${\tiny $\pm7$} & $8${\tiny $\pm11$} & $77${\tiny $\pm9$} & $61${\tiny $\pm2$} \\
        \envname{antmaze-large-navigate-task4-v0} & \datasetname{clean} & $0${\tiny $\pm0$} & $74${\tiny $\pm5$} & $67${\tiny $\pm45$} & $0${\tiny $\pm0$} & $\mathbf{80}${\tiny $\pm6$} & $68${\tiny $\pm12$} \\
        \envname{antmaze-large-navigate-task5-v0} & \datasetname{clean} & $2${\tiny $\pm2$} & $61${\tiny $\pm4$} & $76${\tiny $\pm26$} & $6${\tiny $\pm4$} & $\mathbf{86}${\tiny $\pm9$} & $\mathbf{90}${\tiny $\pm3$} \\
        \midrule
        \envname{antmaze-large-explore-task1-v0} & \datasetname{noisy} & $40${\tiny $\pm30$} & $0${\tiny $\pm0$} & $0${\tiny $\pm0$} & $84${\tiny $\pm7$} & $0${\tiny $\pm0$} & $\mathbf{91}${\tiny $\pm6$} \\
        \envname{antmaze-large-explore-task2-v0} & \datasetname{noisy} & $0${\tiny $\pm0$} & $0${\tiny $\pm0$} & $0${\tiny $\pm0$} & $41${\tiny $\pm13$} & $0${\tiny $\pm0$} & $\mathbf{91}${\tiny $\pm6$} \\
        \envname{antmaze-large-explore-task3-v0} & \datasetname{noisy} & $69${\tiny $\pm28$} & $0${\tiny $\pm0$} & $0${\tiny $\pm0$} & $\mathbf{92}${\tiny $\pm8$} & $0${\tiny $\pm0$} & $\mathbf{87}${\tiny $\pm2$} \\
        \envname{antmaze-large-explore-task4-v0} & \datasetname{noisy} & $36${\tiny $\pm15$} & $0${\tiny $\pm0$} & $0${\tiny $\pm0$} & $\mathbf{56}${\tiny $\pm37$} & $0${\tiny $\pm0$} & $5${\tiny $\pm8$} \\
        \envname{antmaze-large-explore-task5-v0} & \datasetname{noisy} & $0${\tiny $\pm0$} & $0${\tiny $\pm0$} & $0${\tiny $\pm0$} & $16${\tiny $\pm22$} & $0${\tiny $\pm0$} & $\mathbf{88}${\tiny $\pm14$} \\
        \midrule
        \envname{cube-single-play-task1-v0} & \datasetname{clean} & $40${\tiny $\pm13$} & $47${\tiny $\pm5$} & $86${\tiny $\pm2$} & $57${\tiny $\pm12$} & $60${\tiny $\pm43$} & $\mathbf{97}${\tiny $\pm3$} \\
        \envname{cube-single-play-task2-v0} & \datasetname{clean} & $7${\tiny $\pm3$} & $20${\tiny $\pm17$} & $73${\tiny $\pm10$} & $57${\tiny $\pm5$} & $\mathbf{86}${\tiny $\pm3$} & $\mathbf{85}${\tiny $\pm11$} \\
        \envname{cube-single-play-task3-v0} & \datasetname{clean} & $14${\tiny $\pm5$} & $18${\tiny $\pm1$} & $77${\tiny $\pm6$} & $44${\tiny $\pm35$} & $68${\tiny $\pm21$} & $\mathbf{99}${\tiny $\pm1$} \\
        \envname{cube-single-play-task4-v0} & \datasetname{clean} & $30${\tiny $\pm6$} & $19${\tiny $\pm17$} & $73${\tiny $\pm9$} & $77${\tiny $\pm3$} & $59${\tiny $\pm19$} & $\mathbf{89}${\tiny $\pm11$} \\
        \envname{cube-single-play-task5-v0} & \datasetname{clean} & $43${\tiny $\pm12$} & $25${\tiny $\pm19$} & $85${\tiny $\pm14$} & $73${\tiny $\pm3$} & $61${\tiny $\pm28$} & $\mathbf{97}${\tiny $\pm4$} \\
        \midrule
        \envname{cube-single-noisy-task1-v0} & \datasetname{noisy} & $46${\tiny $\pm10$} & $12${\tiny $\pm2$} & $68${\tiny $\pm8$} & $\mathbf{95}${\tiny $\pm1$} & $31${\tiny $\pm22$} & $\mathbf{99}${\tiny $\pm1$} \\
        \envname{cube-single-noisy-task2-v0} & \datasetname{noisy} & $53${\tiny $\pm15$} & $2${\tiny $\pm2$} & $71${\tiny $\pm3$} & $\mathbf{97}${\tiny $\pm1$} & $45${\tiny $\pm10$} & $\mathbf{100}${\tiny $\pm0$} \\
        \envname{cube-single-noisy-task3-v0} & \datasetname{noisy} & $68${\tiny $\pm6$} & $5${\tiny $\pm5$} & $54${\tiny $\pm3$} & $\mathbf{98}${\tiny $\pm1$} & $3${\tiny $\pm1$} & $\mathbf{98}${\tiny $\pm2$} \\
        \envname{cube-single-noisy-task4-v0} & \datasetname{noisy} & $40${\tiny $\pm4$} & $2${\tiny $\pm1$} & $63${\tiny $\pm5$} & $94${\tiny $\pm1$} & $31${\tiny $\pm5$} & $\mathbf{100}${\tiny $\pm1$} \\
        \envname{cube-single-noisy-task5-v0} & \datasetname{noisy} & $37${\tiny $\pm4$} & $3${\tiny $\pm2$} & $72${\tiny $\pm7$} & $\mathbf{96}${\tiny $\pm1$} & $61${\tiny $\pm3$} & $\mathbf{99}${\tiny $\pm1$} \\
        \midrule
        \envname{cube-double-play-task1-v0} & \datasetname{clean} & $42${\tiny $\pm6$} & $7${\tiny $\pm5$} & $37${\tiny $\pm7$} & $32${\tiny $\pm2$} & $68${\tiny $\pm26$} & $\mathbf{74}${\tiny $\pm6$} \\
        \envname{cube-double-play-task2-v0} & \datasetname{clean} & $22${\tiny $\pm10$} & $4${\tiny $\pm1$} & $30${\tiny $\pm3$} & $2${\tiny $\pm3$} & $47${\tiny $\pm33$} & $\mathbf{90}${\tiny $\pm12$} \\
        \envname{cube-double-play-task3-v0} & \datasetname{clean} & $17${\tiny $\pm2$} & $1${\tiny $\pm1$} & $17${\tiny $\pm6$} & $4${\tiny $\pm4$} & $42${\tiny $\pm30$} & $\mathbf{90}${\tiny $\pm7$} \\
        \envname{cube-double-play-task4-v0} & \datasetname{clean} & $30${\tiny $\pm15$} & $11${\tiny $\pm11$} & $25${\tiny $\pm6$} & $4${\tiny $\pm1$} & $30${\tiny $\pm23$} & $\mathbf{90}${\tiny $\pm7$} \\
        \envname{cube-double-play-task5-v0} & \datasetname{clean} & $12${\tiny $\pm5$} & $2${\tiny $\pm1$} & $24${\tiny $\pm10$} & $26${\tiny $\pm7$} & $17${\tiny $\pm23$} & $\mathbf{82}${\tiny $\pm21$} \\
        \midrule
        \envname{cube-double-noisy-task1-v0} & \datasetname{noisy} & $62${\tiny $\pm5$} & $6${\tiny $\pm4$} & $12${\tiny $\pm14$} & $68${\tiny $\pm14$} & $86${\tiny $\pm9$} & $\mathbf{94}${\tiny $\pm6$} \\
        \envname{cube-double-noisy-task2-v0} & \datasetname{noisy} & $5${\tiny $\pm3$} & $2${\tiny $\pm1$} & $2${\tiny $\pm1$} & $37${\tiny $\pm18$} & $\mathbf{76}${\tiny $\pm25$} & $56${\tiny $\pm20$} \\
        \envname{cube-double-noisy-task3-v0} & \datasetname{noisy} & $5${\tiny $\pm3$} & $3${\tiny $\pm2$} & $7${\tiny $\pm6$} & $15${\tiny $\pm4$} & $\mathbf{75}${\tiny $\pm30$} & $52${\tiny $\pm22$} \\
        \envname{cube-double-noisy-task4-v0} & \datasetname{noisy} & $10${\tiny $\pm3$} & $6${\tiny $\pm2$} & $7${\tiny $\pm3$} & $20${\tiny $\pm12$} & $\mathbf{72}${\tiny $\pm31$} & $33${\tiny $\pm47$} \\
        \envname{cube-double-noisy-task5-v0} & \datasetname{noisy} & $8${\tiny $\pm6$} & $6${\tiny $\pm4$} & $2${\tiny $\pm0$} & $9${\tiny $\pm6$} & $\mathbf{90}${\tiny $\pm10$} & $67${\tiny $\pm23$} \\
        \midrule
        \envname{scene-play-task1-v0} & \datasetname{clean} & $\mathbf{99}${\tiny $\pm1$} & $69${\tiny $\pm6$} & $\mathbf{98}${\tiny $\pm2$} & $23${\tiny $\pm17$} & $\mathbf{94}${\tiny $\pm2$} & $\mathbf{95}${\tiny $\pm2$} \\
        \envname{scene-play-task2-v0} & \datasetname{clean} & $41${\tiny $\pm7$} & $8${\tiny $\pm6$} & $73${\tiny $\pm3$} & $27${\tiny $\pm12$} & $87${\tiny $\pm1$} & $\mathbf{99}${\tiny $\pm1$} \\
        \envname{scene-play-task3-v0} & \datasetname{clean} & $56${\tiny $\pm2$} & $52${\tiny $\pm4$} & $89${\tiny $\pm1$} & $20${\tiny $\pm15$} & $85${\tiny $\pm2$} & $\mathbf{99}${\tiny $\pm1$} \\
        \envname{scene-play-task4-v0} & \datasetname{clean} & $24${\tiny $\pm1$} & $15${\tiny $\pm11$} & $65${\tiny $\pm26$} & $46${\tiny $\pm2$} & $\mathbf{97}${\tiny $\pm2$} & $25${\tiny $\pm10$} \\
        \envname{scene-play-task5-v0} & \datasetname{clean} & $80${\tiny $\pm3$} & $65${\tiny $\pm5$} & $76${\tiny $\pm6$} & $22${\tiny $\pm23$} & $92${\tiny $\pm3$} & $\mathbf{99}${\tiny $\pm1$} \\
        \midrule
        \envname{scene-noisy-task1-v0} & \datasetname{noisy} & $87${\tiny $\pm9$} & $27${\tiny $\pm11$} & $87${\tiny $\pm1$} & $\mathbf{100}${\tiny $\pm0$} & $39${\tiny $\pm43$} & $\mathbf{99}${\tiny $\pm1$} \\
        \envname{scene-noisy-task2-v0} & \datasetname{noisy} & $40${\tiny $\pm7$} & $1${\tiny $\pm1$} & $23${\tiny $\pm5$} & $\mathbf{95}${\tiny $\pm4$} & $15${\tiny $\pm4$} & $61${\tiny $\pm14$} \\
        \envname{scene-noisy-task3-v0} & \datasetname{noisy} & $66${\tiny $\pm3$} & $3${\tiny $\pm3$} & $66${\tiny $\pm4$} & $\mathbf{96}${\tiny $\pm3$} & $45${\tiny $\pm15$} & $69${\tiny $\pm3$} \\
        \envname{scene-noisy-task4-v0} & \datasetname{noisy} & $57${\tiny $\pm4$} & $5${\tiny $\pm7$} & $53${\tiny $\pm6$} & $\mathbf{88}${\tiny $\pm8$} & $25${\tiny $\pm10$} & $71${\tiny $\pm0$} \\
        \envname{scene-noisy-task5-v0} & \datasetname{noisy} & $57${\tiny $\pm21$} & $59${\tiny $\pm1$} & $70${\tiny $\pm2$} & $\mathbf{96}${\tiny $\pm1$} & $24${\tiny $\pm17$} & $\mathbf{98}${\tiny $\pm2$} \\
        \bottomrule
    \end{tabular}
\end{table}

\paragraph{Full results.} We present the full per-task results of all $40$ tasks in \Cref{tab: full result}. The results are averaged over $3$ seeds and $32$ runs per seed. The results are bolded if the algorithm achieves at or above $95\%$ of the best performance following \citet{ogbench_park2025}.

\begin{figure}[t!]
    \centering
    \includegraphics[width=\linewidth]{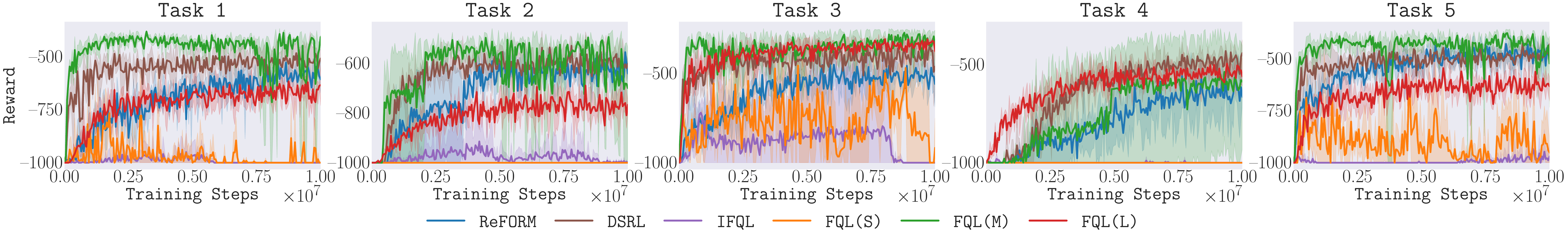}
    \caption{Training curves for \envname{antmaze-large} environment with the \datasetname{clean} dataset.}
    \label{fig: antmaze-large-navigate-training}
\end{figure}

\begin{figure}[t!]
    \centering
    \includegraphics[width=\linewidth]{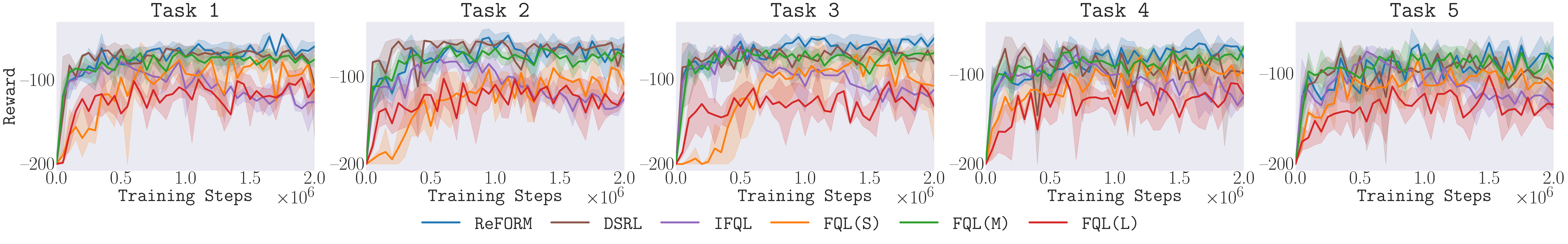}
    \caption{Training curves for \envname{cube-single} environment with the \datasetname{clean} dataset.}
    \label{fig: cube-single-play-training}
\end{figure}

\begin{figure}[t!]
    \centering
    \includegraphics[width=\linewidth]{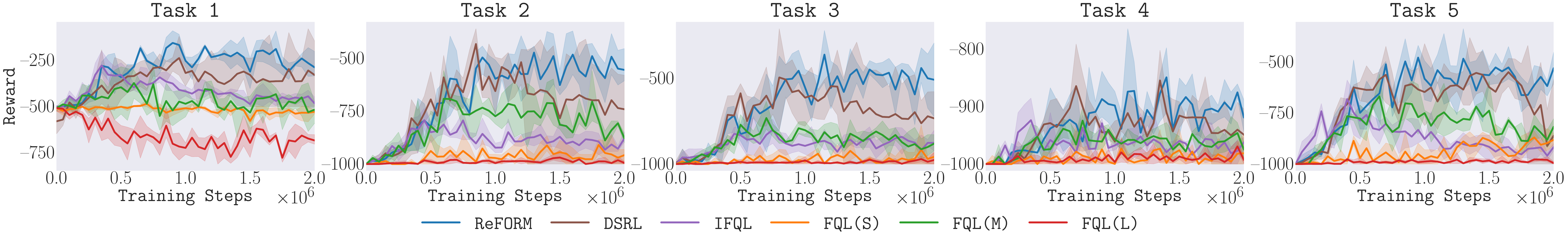}
    \caption{Training curves for \envname{cube-double} environment with the \datasetname{clean} dataset.}
    \label{fig: cube-double-play-training}
\end{figure}

\begin{figure}[t!]
    \centering
    \includegraphics[width=\linewidth]{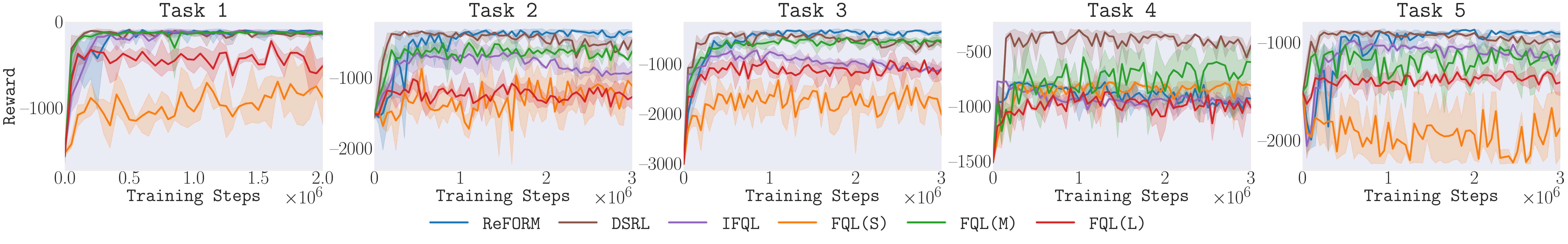}
    \caption{Training curves for \envname{scene} environment with the \datasetname{clean} dataset.}
    \label{fig: scene-play-training}
\end{figure}

\begin{figure}[t!]
    \centering
    \includegraphics[width=\linewidth]{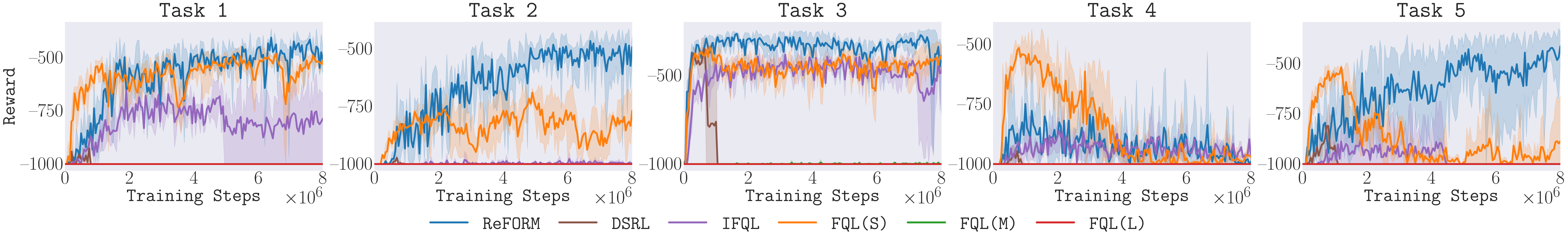}
    \caption{Training curves for \envname{antmaze-large} environment with the \datasetname{noisy} dataset.}
    \label{fig: antmaze-large-explore-training}
\end{figure}

\begin{figure}[t!]
    \centering
    \includegraphics[width=\linewidth]{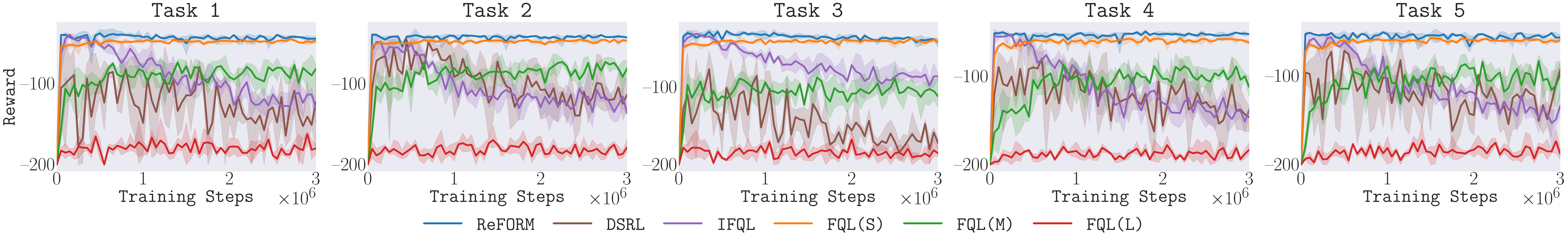}
    \caption{Training curves for \envname{cube-single} environment with the \datasetname{noisy} dataset.}
    \label{fig: cube-single-noisy-training}
\end{figure}

\begin{figure}[t!]
    \centering
    \includegraphics[width=\linewidth]{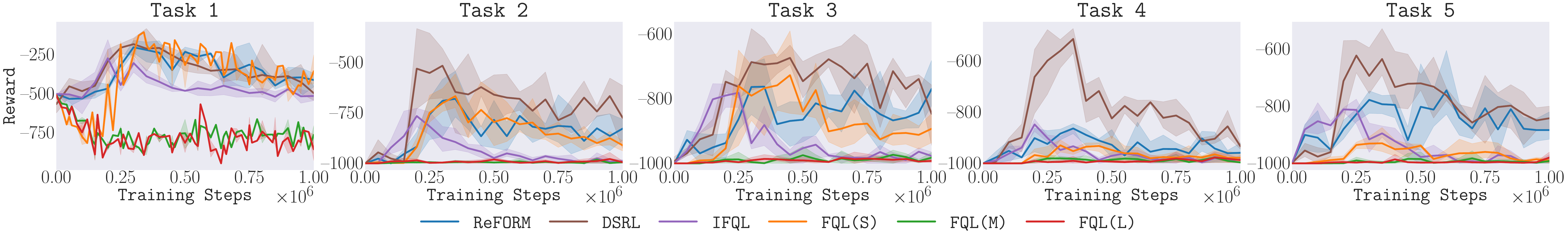}
    \caption{Training curves for \envname{cube-double} environment with the \datasetname{noisy} dataset.}
    \label{fig: cube-double-noisy-training}
\end{figure}

\begin{figure}[t!]
    \centering
    \includegraphics[width=\linewidth]{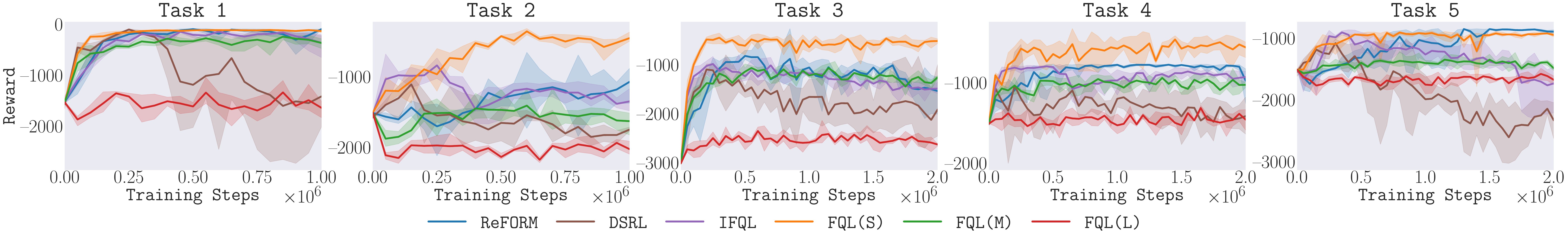}
    \caption{Training curves for \envname{scene} environment with the \datasetname{noisy} dataset.}
    \label{fig: scene-noisy-training}
\end{figure}

\paragraph{Training curves.} We present the training curves for all tasks in all environments with both the \datasetname{clean} dataset (\Cref{fig: antmaze-large-navigate-training}-\Cref{fig: scene-play-training}) and the \datasetname{noisy} dataset (\Cref{fig: antmaze-large-explore-training}-\Cref{fig: scene-noisy-training}). In addition, we present training curves corresponding to \Cref{fig: ablations} (left) in the main pages in \Cref{fig: ablation-training}.

\paragraph{D4RL results.} We further conduct experiments in D4RL \citep{fu2020d4rl} antmaze and adoit environments to test the performance of \ourname{} across different benchmarks. Although the benchmark is different, we \textit{maintain} the hyperparameters of \ourname{} as those in OGBench. The results (D4RL normalized return) are shown in \Cref{tab: d4rl}, and \ourname{} still consistently achieves the best or comparable results. \baselinename{DSRL} is omitted in these results because its paper \citep{wagenmaker2025steering} does not report the best hyperparameters of \baselinename{DSRL} in these environments. 

\begin{table}
	\begin{center}
    \caption{\textbf{D4RL results.} We present the following results on environments in the D4RL \cite{fu2020d4rl} benchmark. The results are averaged over $3$ seeds and $32$ runs per seed. The results are bolded if the algorithm achieves at or above $95\%$ of the best performance following \citet{ogbench_park2025}.}
    \label{tab: d4rl}
	\begin{tabular}{l|cccc|c}
			\toprule
            Environment & \baselinename{IFQL} & \baselinename{FQL(L)} & \baselinename{FQL(M)} & \baselinename{FQL(S)} & \ourname{} \\
			\midrule
            \envname{antmaze-umaze-v2} & $91${\tiny $\pm7$} & $85${\tiny $\pm4$} & $\mathbf{99}${\tiny $\pm1$} & $88${\tiny $\pm13$} & $\mathbf{97}${\tiny $\pm0$} \\
			\envname{antmaze-umaze-diverse-v2} & $55${\tiny $\pm28$} & $57${\tiny $\pm10$} & $\mathbf{88}${\tiny $\pm5$} & $61${\tiny $\pm26$} & $\mathbf{83}${\tiny $\pm3$} \\
			\envname{antmaze-medium-play-v2} & $3${\tiny $\pm4$} & $14${\tiny $\pm6$} & $\mathbf{92}${\tiny $\pm1$} & $52${\tiny $\pm15$} & $85${\tiny $\pm4$} \\
			\envname{antmaze-medium-diverse-v2} & $24${\tiny $\pm34$} & $9${\tiny $\pm4$} & $\mathbf{81}${\tiny $\pm13$} & $24${\tiny $\pm30$} & $\mathbf{80}${\tiny $\pm4$} \\
			\envname{antmaze-large-play-v2} & $17${\tiny $\pm21$} & $43${\tiny $\pm10$} & $61${\tiny $\pm21$} & $3${\tiny $\pm4$} & $\mathbf{71}${\tiny $\pm4$} \\
			\envname{antmaze-large-diverse-v2} & $28${\tiny $\pm27$} & $55${\tiny $\pm4$} & $\mathbf{85}${\tiny $\pm8$} & $8${\tiny $\pm12$} & $69${\tiny $\pm9$} \\
            \midrule
			\envname{pen-human-v1} & $\mathbf{65}${\tiny $\pm1$} & $48${\tiny $\pm0$} & $59${\tiny $\pm4$} & $31${\tiny $\pm4$} & $\mathbf{64}${\tiny $\pm7$} \\
			\envname{pen-cloned-v1} & $\mathbf{81}${\tiny $\pm8$} & $61${\tiny $\pm7$} & $66${\tiny $\pm5$} & $57${\tiny $\pm6$} & $70${\tiny $\pm6$} \\
			\envname{pen-expert-v1} & $120${\tiny $\pm3$} & $105${\tiny $\pm7$} & $\mathbf{128}${\tiny $\pm1$} & $107${\tiny $\pm10$} & $\mathbf{129}${\tiny $\pm7$} \\
            \envname{door-human-v1} & $3${\tiny $\pm1$} & $2${\tiny $\pm1$} & $0${\tiny $\pm0$} & $0${\tiny $\pm0$} & $\mathbf{4}${\tiny $\pm1$} \\
			\envname{door-cloned-v1} & $-0${\tiny $\pm0$} & $0${\tiny $\pm0$} & $\mathbf{3}${\tiny $\pm2$} & $0${\tiny $\pm0$} & $1${\tiny $\pm1$} \\
			\envname{door-expert-v1} & $89${\tiny $\pm5$} & $\mathbf{104}${\tiny $\pm1$} & $\mathbf{105}${\tiny $\pm0$} & $\mathbf{102}${\tiny $\pm0$} & $\mathbf{104}${\tiny $\pm4$} \\
			\bottomrule
		\end{tabular}
	\end{center}
\end{table}

\paragraph{Visual manipulation results.} We also conduct experiments in OGBench \citep{ogbench_park2025} visual manipulation environments to test the performance of \ourname{} with higher-dimensional image-based inputs. Similarly, we \textit{maintain} the hyperparameters of \ourname{}. The results (return) are shown in \Cref{tab: visual}, and \ourname{} performs the best. \baselinename{DSRL} is omitted because its best hyperparameters in these environments are not reported in \citet{wagenmaker2025steering}.

\begin{table}
	\begin{center}
    \caption{\textbf{Visual manipulation results.} We present the following results on visual manipulation environments in OGBench \cite{ogbench_park2025}. The results are averaged over $3$ seeds and $32$ runs per seed. The results are bolded if the algorithm achieves at or above $95\%$ of the best performance following \citet{ogbench_park2025}. To save space, the \envname{-singletask} tags are omitted from task names.}
    \label{tab: visual}
    \scriptsize
	\begin{tabular}{l|c|cccc|c}
			\toprule
            Task & Dataset & \baselinename{IFQL} & \baselinename{FQL(L)} & \baselinename{FQL(M)} & \baselinename{FQL(S)} & \ourname{} \\
			\midrule
            \envname{visual-cube-single-play-task1-v0} & \datasetname{clean} & $-117${\tiny $\pm7$} & $-150${\tiny $\pm16$} & $\mathbf{-110}${\tiny $\pm9$} & $-138${\tiny $\pm19$} & $\mathbf{-108}${\tiny $\pm12$} \\
			\envname{visual-cube-single-noisy-task1-v0} & \datasetname{noisy} & $-95${\tiny $\pm2$} & $-176${\tiny $\pm10$} & $-103${\tiny $\pm2$} & $-57${\tiny $\pm3$} & $\mathbf{-52}${\tiny $\pm7$} \\
			\bottomrule
		\end{tabular}
	\end{center}
\end{table}

\begin{figure}[t!]
    \centering
    \begin{subfigure}{.49\linewidth}
        \centering
        \includegraphics[width=\linewidth]{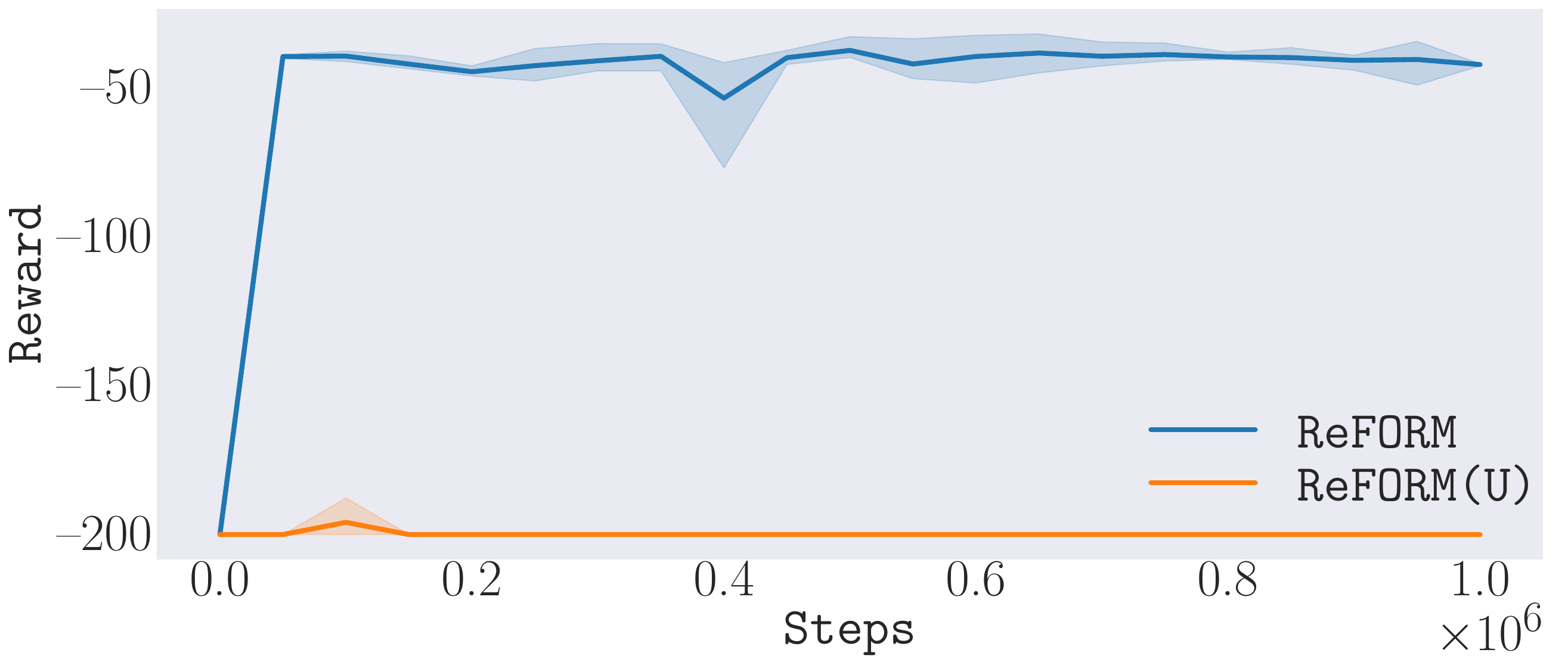}
        \caption{Comparing \algo{} and \baselinename{ReFORM(U)}.}
    \end{subfigure}
\begin{subfigure}{.49\linewidth}
        \centering
        \includegraphics[width=\linewidth]{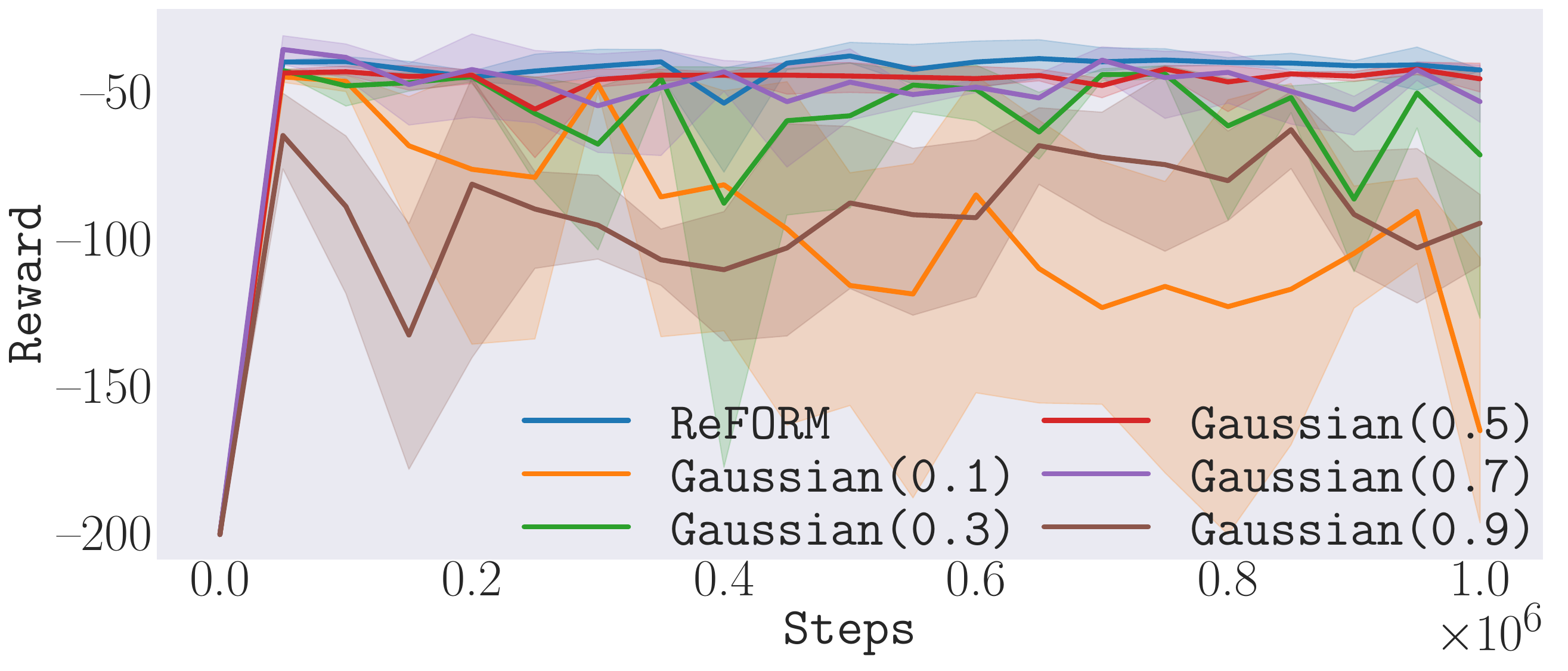}
        \caption{Comparing \algo{} and \baselinename{Gaussian($\xi$)}.}
    \end{subfigure}
    \caption{\textbf{More training curves for ablation studies.} \baselinename{ReFORM(U)} changes both $q_\mathrm{NG}$ and $q_\mathrm{BC}$ from $\unif(\gB_l^d)$ to the standard Gaussian distribution and removes the reflection term. \baselinename{Gaussian($\xi$)} keeps $q_\mathrm{NG}=\unif(\gB_l^d)$ but changes $q_\mathrm{BC}$ to the standard Gaussian distribution. Then, the radius of the hypersphere $\gB_l^d$ is chosen such that the standard Gaussian distribution has probability mass $\xi$ in $\gB_l^d$.}
    \label{fig: ablation-training}
\end{figure}

\paragraph{Visualization of the generated noise in the toy example.}
For the toy example presented in \Cref{sec: ablations}, it is possible to visualize the generated noises directly for \ourname{} and \baselinename{DSRL}. We visualize the noises in \Cref{fig: viz-noise-toy}. We observe that the generated noise with reflected flow (\ourname{}) is more concentrated while retaining two modes, while with a Gaussian distribution squashed by $\mathrm{tanh}$ (\baselinename{DSRL}), the generated noise is unimodal and spreads out a lot. 

\begin{figure}
    \centering
    \begin{subfigure}{.3\linewidth}
        \centering
        \includegraphics[width=\linewidth]{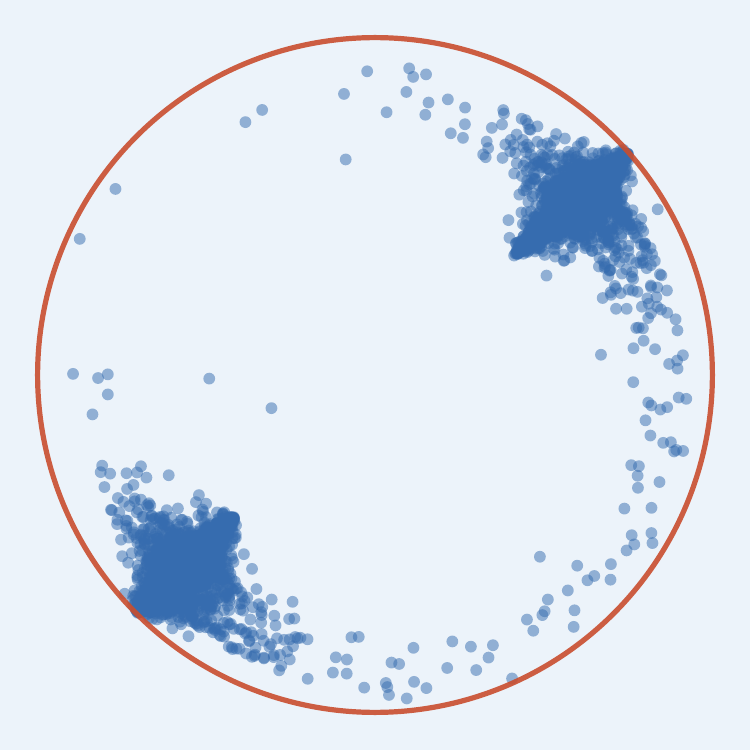}
        \caption{Generated noise of \ourname{}.}
    \end{subfigure}
        \begin{subfigure}{.3\linewidth}
        \centering
        \includegraphics[width=\linewidth]{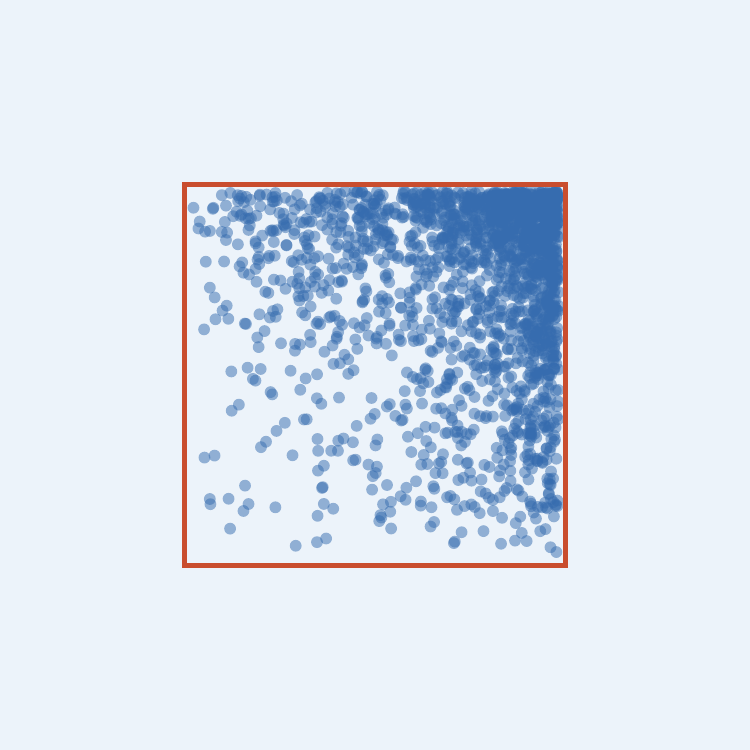}
        \caption{Generated noise of \baselinename{DSRL}.}
    \end{subfigure}
    \caption{Visualization of the generated noises in the 2-dimensional toy example. }
    \label{fig: viz-noise-toy}
\end{figure}

\paragraph{Ablations on the radius of the hypersphere $\gB_l^d$.} 
One hyperparameter introduced in \ourname{} is the radius $l$ of the hypersphere $\gB_l^d$. As discussed in \Cref{app: implementation}, we select the smallest $l$ such that $\gA\subseteq \gB_l^d$ in our implementation. In OGBench environments, the action space is $[-1, 1]^d$, so we choose $l=\sqrt{d}$. To study the sensitivity of \ourname{} w.r.t. $l$, we conduct experiments in the \envname{cube-single} environment with the \datasetname{noisy} dataset, and vary $l$ in $\{0.25\sqrt{d}, 0.5\sqrt{d}, \sqrt{d}, 2\sqrt{d}, 4\sqrt{d}\}$. The results are shown in \Cref{fig: ablation r_max}, which shows no significant difference among these choices. Therefore, \ourname{} is robust w.r.t. the choice of $l$, and empirically $l$ can be chosen as any number close to the scale of the action space. 

\paragraph{Training time.} We report the training time of \ourname{} and all baselines in \Cref{tab: training-time}. The table shows that \ourname{} indeed doubles the training time compared to FQL due to the 2-stage flow. However, as shown in our experiments, \baselinename{FQL} is sensitive to hyperparameters, and searching for optimal hyperparameters requires significantly more runs. On the contrary, \ourname{} can be used without any hyperparameter searching.

\begin{figure}[t]
    \centering
    \includegraphics[width=0.49\linewidth]{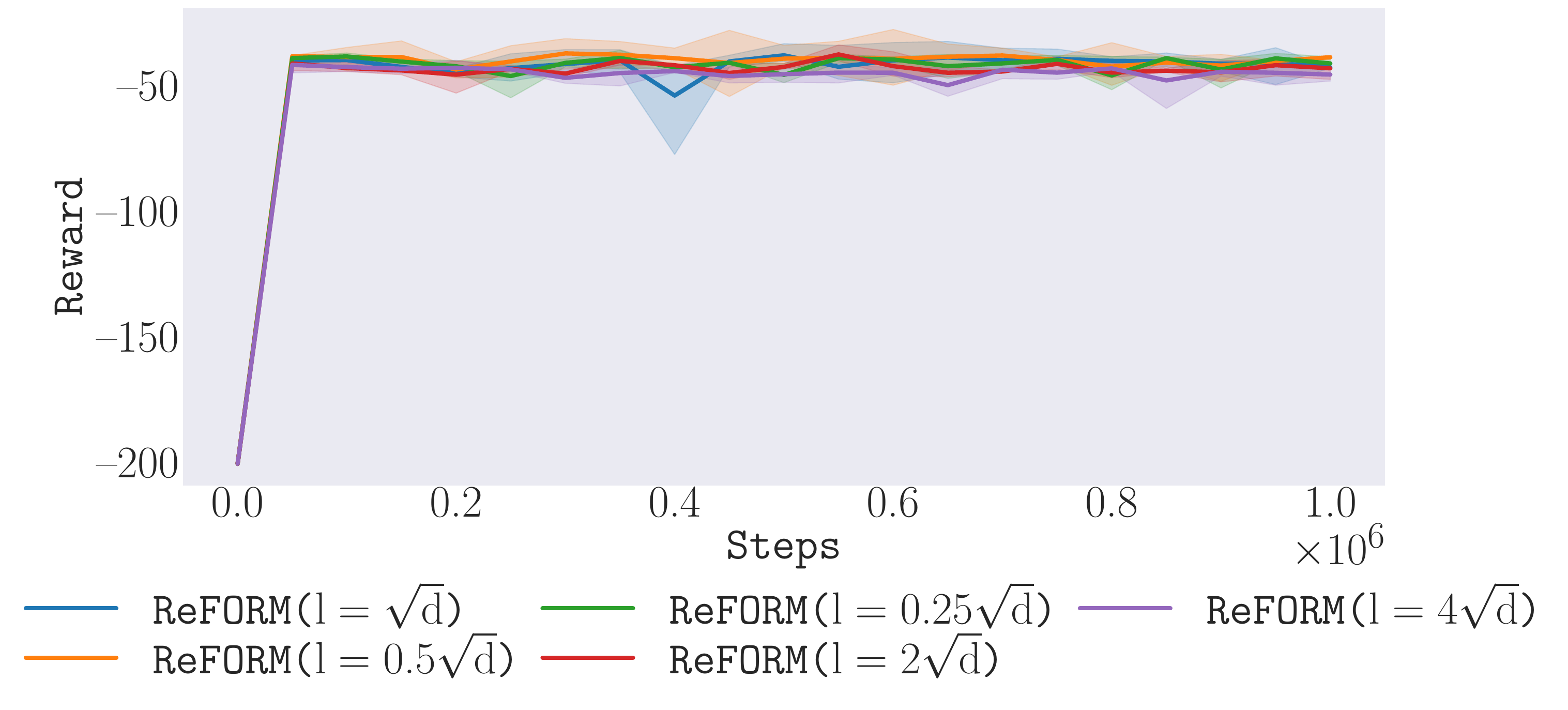}
    \caption{Sensitivity analysis of \algo{} w.r.t. $l$.}
    \label{fig: ablation r_max}
\end{figure}

\begin{table}[t]
    \centering
    \caption{Approximate training time of \algo{} and the baselines.}
    \label{tab: training-time}
    \begin{tabular}{l|cccc}
        \toprule
        Algorithm & \algo{} & \baselinename{FQL} & \baselinename{IFQL} & \baselinename{DSRL} \\
        \midrule
        Training time (minutes, $10^6$ steps) & $80$ & $40$ & $35$ & $55$\\
        \bottomrule
    \end{tabular}
\end{table}

\subsection{Code}

We provide code for \algo{} in our supplementary materials. 

\section{The Use of Large Language Models}

This paper uses Large Language Models to correct spelling and grammar issues. 

\end{document}